\pdfoutput=1
\documentclass[11pt]{article}
\usepackage[font=libertinus, citestyle=authoryear]{kurbanlab}

\DeclareAffiliation{hbku}{%
  College of Science and Engineering, Hamad Bin Khalifa University, Doha, Qatar}

\DeclareAffiliation{tamu}{%
  Department of Computer and Electrical Engineering,
  Texas A\&M University, College Station, TX, USA}

\DeclareAffiliation{iub}{%
  Luddy School of Informatics, Computing, and Engineering,
  Indiana University Bloomington, Bloomington, IN, USA}

\DeclareAffiliation{uis}{%
  Department of Computer Science,
  University of Illinois Springfield, Springfield, IL, USA}

\usepackage{booktabs}
\usepackage{colortbl}
\newcommand{\tabstyle}{\small
  \renewcommand{\arraystretch}{1.0}%
  \setlength{\aboverulesep}{0pt}\setlength{\belowrulesep}{0pt}}
\newcommand{\hdrstrut}{\rule[-0.85ex]{0pt}{3.5ex}}
\newcommand{\hdr}[1]{{\bfseries\color{kilink}\hdrstrut #1}}

\definecolor{helpblue}{HTML}{0072B2}
\definecolor{hurtrust}{HTML}{D55E00}

\title{When Does Retrieval Help Time-Series Forecasting?}
\RunningTitle{When does retrieval help time-series forecasting?}

\Author{Mert Onur Cakiroglu}{iub}
\Author{Elham Buxton}{uis}
\Author{Mehmet Dalkilic}{iub}
\Author[corresponding=hkurban@hbku.edu.qa, orcid=0000-0003-3142-2866]{Hasan Kurban}{hbku}

\CodeURL{https://github.com/KurbanIntelligenceLab/retrieval-regime}
\Venue{Under review}

\begin{document}
\maketitle

\begin{abstract}
Retrieval plug-ins supply a deep forecaster with information its lookback window cannot carry. Published evaluations report consistent gains, and each credits its own mechanism. We show that the benefit belongs instead to the operating point: the relation between window length $S$ and dominant seasonal period $L$, an axis the standard protocol never varies. Stratifying the evaluation by that relation exposes the regime. At $S{=}12$, a simple control that repeats the last observed period beats the six standard backbones, in aggregate, on four of seven benchmarks by $8\%$ to $44\%$ of MSE. It beats the strongest plug-in we run on ETTm1 and matches it on ECL. It is worse by up to $25\%$ on the three datasets whose training-split spectra lack a concentrated, shared period. A controlled synthetic sweep of horizon, period, and window shows the benefit boundary tracks the period (correlation $+0.71$), not the horizon ($-0.23$). A paired control with no phase to recover nearly erases the effect, consistent with phase starvation. Zero-shot pretraining does not escape it: a foundation model trails trained backbones by $22\%$ to $50\%$ on the periodic benchmarks. Within our instrument, exact lookup matches graph diffusion: the payoff is consulting the record, not the machinery on top. Two interpretable statistics, a trend test and a staleness rate, predict the sign of the per-cell benefit at $0.76$ accuracy under leave-one-dataset-out evaluation, a suggestive margin over the $0.69$ majority rule, where a 22-feature stack manages $0.57$. We propose no new plug-in. The contribution is the regime map, the protocol that reveals it, and two statistics that screen it before deployment.
Code: \url{https://github.com/KurbanIntelligenceLab/retrieval-regime}.
\end{abstract}

\printkeywords

\section{Introduction}

Retrieval plug-ins are an increasingly common companion to deep time-series forecasters. They inject structure from the training record that the window cannot carry: FAN \citep{ye2024fan} through a learned frequency prior, GTR \citep{cao2026gtr} through a learned cycle embedding, and RAFT \citep{han2025raft} by retrieving raw training windows. All three report consistent gains at their published operating points and credit their own mechanisms. FAN and GTR fix the lookback at $S{=}96$, while RAFT tunes it to 720 on the benchmarks we share.

While such evaluations ask whether a mechanism improves its chosen configuration, deployment asks whether retrieval of any kind helps at the operating point in hand. These questions differ because retrieval supplies cross-window structure, particularly information about the phase of the dominant seasonal cycle. A window covering a full period may contain this information on its own, whereas a shorter window may leave the phase unresolved. Whether retrieval has useful information to add depends mainly on the window \(S\) relative to the period \(L\), while the horizon \(H\) plays a secondary role. The mechanism determines how effectively that information is used. Consequently, we pose the question: \emph{When does retrieval help time-series forecasting?}

\begin{figure}[t]
\centering
\includegraphics[width=0.73\textwidth]{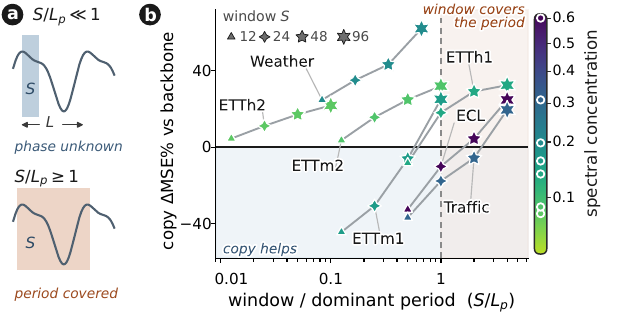}
\caption{The paper in one figure. (a) The operating point: a short window $S$ cannot fix the phase of period $L$, a covering one can. (b) Per-benchmark, per-window points, the period tile's benefit over the six trained backbones (paired $\Delta$MSE\%; $S{=}12$ as in Table~\ref{tab:copy}, larger windows from the sweep) against the covered fraction $S/L_p$ (color: spectral concentration, the share of detrended training-split power at the dominant period; size: window). Concentrated datasets gain up to 44\% when the window cannot cover the period and lose the gain as $S/L_p$ grows; weakly concentrated or heterogeneous ones never gain.}
\label{fig:teaser}
\end{figure}

\begin{figure*}[t]
\centering
\includegraphics[width=\textwidth]{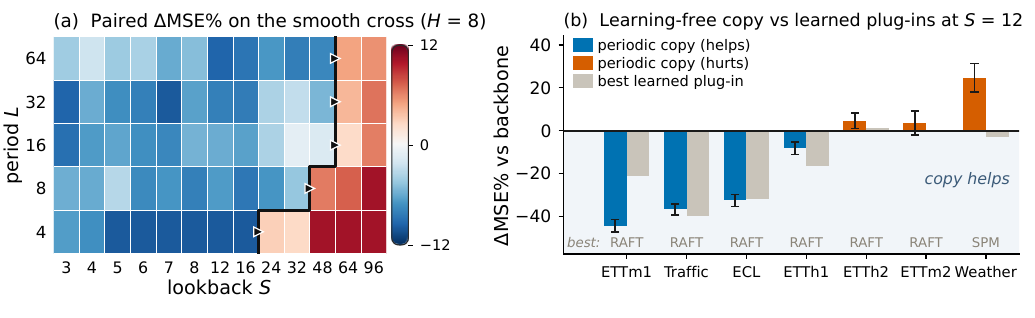}
\caption{The operating regime, in synthetic and real form. (a) Paired $\Delta$MSE\% of our probe, the symbolic periodic memory (SPM), against its DLinear backbone on the smooth synthetic cross at $H{=}8$. Blue means retrieval helps; each marker ($\blacktriangleright$) is the per-row zero crossing, and the helpful region widens as the period $L$ grows. (b) The period tile against the best learned plug-in per dataset at $S{=}12$, as paired $\Delta$MSE\% versus six trained backbones over four horizons and three seeds. Tile bars: 95\% CIs across the twenty-four backbone-horizon cells. The tile helps where the training-split periodogram shows a strong period and fails where the spectrum is trend dominated (ETTh2), weakly structured (ETTm2), or heterogeneous (Weather).}
\label{fig:regime}
\end{figure*}

The standard protocol cannot isolate this effect because \(S=96\) already spans the dominant daily cycle on most benchmarks. Weather, ECL, and Traffic are the main exceptions: their 144- or 168-step cycles extend beyond the lookback, so retrieval can still add information the window lacks. Recent benchmark critiques similarly argue that rigid task formulations obscure such structure \citep{qiao2026time}. Figure~\ref{fig:teaser} reveals it by relating retrieval benefit to the ratio of lookback length to the dominant period.

Varying the window makes the regime visible. At $S{=}12$, we use a simple parameter-free control, the \emph{period tile}, which repeats the most recent complete seasonal period (Figure~\ref{fig:regime}(b)). Aggregated across the six deep backbones, it reduces MSE by 8\% to 44\% on four of seven benchmarks. It also outperforms the strongest plug-in we test on ETTm1 and performs comparably to it on ECL, but provides no benefit on the remaining three benchmarks, increasing MSE by as much as 25\%. The results align with the training-split spectra: the tile helps when a strong, shared period extends beyond the lookback, but not when periodic structure is weak or trend-dominated (ETTm2, ETTh2) or when dominant periods differ across channels (Weather).

The same pattern holds across the window sweep: as the window grows, the benefit of copying fades and eventually becomes a small loss, with longer-period datasets reaching the crossing at larger windows. A controlled synthetic cross of 1{,}656 runs varies the period, horizon, and window independently and shows that the crossing depends on the period rather than the horizon. Replacing the smooth cycle with discrete motifs that recur at random times removes any fixed position within the cycle, and the benefit nearly disappears. We call this position the phase, and the condition in which the window is too short to identify it \emph{phase starvation}.

At this phase-starved operating point, three published plug-ins, our probe, and a zero-shot foundation model succeed or fail according to the same regime, regardless of their sophistication. Two interpretable statistics help identify the regime before deployment, while each mechanism's design determines the magnitude of the benefit.

We make three contributions. \textbf{First, a regime-stratified evaluation.} We introduce a protocol that varies the lookback $S$ relative to a spectral estimate of the dominant period $L$ across seven benchmarks and six backbones, together with a controlled synthetic $H \times L \times S$ study and a phase-free control. \textbf{Second, a replicated empirical finding.} Across the mechanisms we test, retrieval benefit is governed primarily by the window-to-period relation rather than the forecast horizon. On real data, a parameter-free period tile outperforms trained backbones on four of seven benchmarks at $S{=}12$, while zero-shot pretraining remains subject to the same regime. The pattern also transfers to withheld datasets and, along the concentration axis, to retail forecasting, and Proposition~\ref{prop:phase} provides a corresponding lower bound for window-only prediction. \textbf{Third, a pre-deployment diagnostic.} Spectral concentration identifies the coarse regime, while two interpretable statistics for trend and staleness predict whether retrieval will help at the cell level and, under leave-one-dataset-out evaluation, transfer better than generic feature-based baselines. Proposition~\ref{prop:risk} connects these statistics to a deployment risk bound.

\section{An Anatomy of Retrieval Plug-ins}

\paragraph{Mechanisms} Existing plug-ins differ along three axes. \emph{What is stored}: a learned prior over instance-extracted frequencies (FAN \citep{ye2024fan}), dataset-level reference frequencies (MFRS \citep{mfrs2025}), a learned per-position cycle embedding (GTR \citep{cao2026gtr}, CycleNet \citep{cyclenet2024}), or the raw training corpus (RAFT \citep{han2025raft}), compressed to cluster medoids in PFRP \citep{du2026pfrp}. \emph{How it is keyed}: by frequency content, by timestamp modulo a cycle length, or by window similarity, jointly across channels or per channel \citep{kang2026craft}. \emph{How it is fused}: by normalization, an additive prior, gated combination \citep{tsmemory2026}, or diffusion guidance \citep{ratd2024}. Despite these differences, all of these methods inject information that the current window does not contain. Pretraining on large corpora serves a related role \citep{ansari2024chronos} and can itself be combined with retrieval \citep{tsrag2025}. We therefore treat retrieval plug-ins as one class based on their shared function.

\paragraph{What prior evaluations condition on} RAFT analyzes retrieval gains in relation to pattern rarity and process autocorrelation, but only on synthetic data at a fixed lookback. Concurrent work modulates retrieval diversity using dataset-level stationarity \citep{zhou2026saraf} or relates retrieval benefit to series smoothness \citep{ride2026}. None of these analyses varies $S$ relative to $L$. At the other extreme, \citet{ahuja2026retrieval} show retrieval outperforming very long contexts, with $S$ ranging from 720 to 3{,}000, where the main challenge is context noise rather than phase starvation.

\paragraph{Window-length theory and practice} Concurrent theory relates the required lookback to a process's memory length \citep{butera2026context}, but this characterization can be loose for periodic signals. Our analysis instead focuses on the complementary role of the dominant period. Because lookback mis-specification can reverse model rankings \citep{lookbackbias2025}, our protocol varies \(S\) as the main experimental factor while holding the tuning budget fixed within each cell.

\paragraph{Simple baselines and diagnostics} Simple baselines have repeatedly recalibrated time-series forecasting \citep{zeng2023transformers,tan2024llm,sun2022fredo}, and seasonal naive remains a standard reference \citep{gifteval2024,hyndman2018forecasting}. Context parroting \citep{zhang2025parroting} shows that copying from a sufficiently long input context can outperform foundation models on dynamical systems. The period tile addresses the complementary case by using the most recent observed period when it lies outside the model's lookback. On the diagnostic side, FFORMS, FFORMA, and related methods use generic series features to predict the best forecaster, combination weights, or method-specific error \citep{talagala2023fforms,montero2020fforma,talagala2022fformpp,abdallah2025autoforecast}. Recent evidence questions whether such features transfer across datasets \citep{akinci2026selection}. Our results support this concern for generic feature sets, while showing that statistics tied to retrieval's failure modes transfer better.

\section{Study Design}

\paragraph{Setting and notation} A forecaster maps a lookback window $X_{t-S+1:t} \in \mathbb{R}^{S \times D}$ from a $D$-channel series to the next $H$ steps. For each channel, we estimate the dominant period $L$ from the strongest peak in the linearly detrended training-split periodogram over periods 2 to 1{,}024. Together, the lookback $S$, period $L$, and horizon $H$ define the \emph{operating point} $(S,L,H)$. We define a \emph{regime} as a region where the sign and magnitude of retrieval benefit remain qualitatively similar. Our \emph{regime claim} is that retrieval helps most when $S \ll L$, weakens as $S$ approaches $L$, and becomes negligible or harmful once the window spans the period. The horizon $H$ changes the magnitude of the effect, but not where this boundary lies. We summarize channel-level periods by their dataset-level median $L_p$, except where noted.

\begin{table}[htbp]
\centering
\tabstyle
\setlength{\tabcolsep}{7pt}

\begin{tabular}{
@{}lrr
lrr@{}
}
\toprule
\rowcolor{kilmist}  \multicolumn{3}{@{}c}{\hdr{\textit{ETT family}}} & \multicolumn{3}{c@{}}{\hdr{\textit{Other datasets}}} \\
\cmidrule(r{0.6em}){1-3}
\cmidrule(l{0.6em}){4-6}
\rowcolor{kilmist}  \hdr{Dataset} & \hdr{$L_p$} & \hdr{Conc.} & \hdr{Dataset} & \hdr{$L_p$} & \hdr{Conc.} \\
\midrule
ETTh1 & 24  & 0.136 & ECL     & 24  & 0.602 \\
ETTh2 & 960 & 0.082 & Traffic & 24  & 0.310 \\
ETTm1 & 96  & 0.160 & Weather & 144 & 0.198 \\
ETTm2 & 96  & 0.089 &         &     &       \\
\bottomrule
\end{tabular}

\caption{Dataset-level dominant period $L_p$ and spectral concentration (Conc.). $L_p$ is the median channel-level periodogram peak, except for ETTh2, where we report the modal trend-scale peak. Concentration is the median channel-level fraction of detrended power at the dataset's modal period. ETTm2 and ETTh2 have the lowest concentration and are the only datasets where no copy or retrieval mechanism reliably helps.}
\label{tab:spectra}
\end{table}

\paragraph{Benchmarks, backbones, protocol} We evaluate seven standard benchmarks (ETTh1, ETTh2, ETTm1, ETTm2, Weather, ECL, and Traffic) using six backbones: DLinear \citep{zeng2023transformers}, PatchTST \citep{nie2023patchtst}, iTransformer \citep{liu2024itransformer}, TimesNet \citep{wu2023timesnet}, TimeMixer \citep{wang2024timemixer}, and TimeXer \citep{wang2024timexer}. We use forecast horizons $H \in \{96,192,336,720\}$ and three seeds. We report mean squared error (MSE) on z-score-normalized data. Mean absolute error (MAE) gives the same conclusion for every in-regime win, although small out-of-regime effects can change sign under MAE (technical appendix). We make paired comparisons within the same backbone, horizon, and seed. For each pair, a mechanism's effect is the percentage change in test MSE relative to the backbone without the mechanism. We report the mean and 95\% Student-$t$ confidence interval across the twenty-four backbone-horizon cell means, averaged over three seeds. At the headline operating point, $S{=}12$ spans at most half of the representative period on every benchmark. On the four benchmarks with $L_p \ge 96$, it covers no more than 12.5\% of the period. Additional window sweeps use $S \in \{24,48,96\}$, while analyses outside the headline tables use $H \in \{96,336\}$. For tractability, the $S{=}12$ experiments on ECL and Traffic use the first 20 channels in every comparison. For the published plug-ins, FAN uses a validation-selected top-4 frequency budget, GTR uses a fixed 168-step cycle (weekly for hourly data), and RAFT uses its published period set with a fixed top-20 retrieval budget.

\paragraph{Reading the periodogram} When a clear calendar cycle is present, the periodogram identifies it as $L$. We also measure how strongly each channel follows the dataset's most common period. We call this \emph{spectral concentration} and compute it as the fraction of detrended periodogram power at that period. Concentration distinguishes benchmarks that share the same nominal cycle but differ in periodic strength (Table~\ref{tab:spectra}). For example, ETTh2's modal peak near 960 steps reflects trend rather than seasonality, while ETTm2 has little power at its nominal daily period. Weather instead has channel-specific periods ranging from 144 to beyond 900 steps, so no single lag fits all channels. These cases explain the tile's failures and Weather's preference for a per-channel mechanism. Concentration therefore provides a coarse regime axis, complemented later by diagnostics for trend and staleness.

\paragraph{The embarrassingly simple control} The period tile repeats the most recent fully observed period, forecasting $\hat{y}_{t+h}=y_{t-L+1+((h-1)\bmod L)}$. This is the causal multi-period extension of the textbook seasonal-naive forecast \citep{hyndman2018forecasting}. By contrast, the common shorthand $\hat{y}_{t+h}=y_{t+h-L}$ is valid only for $h\leq L$; beyond one period, it indexes future targets and leaks ground truth (technical appendix). We first evaluate the tile with one calendar period per dataset. We use the daily cycle for ETT and Weather and the 168-step weekly cycle for ECL and Traffic. We then let the lag vary by channel and select it on the validation split with a causal tiling criterion.

\paragraph{Symbolic periodic memory} The tile tests whether periodic information is useful, but not whether a learned mechanism can exploit it. We therefore use symbolic periodic memory (SPM), a cheap, per-channel, switchable mechanism whose failures remain interpretable. SPM discretizes each channel into $V{=}5$ bins and indexes every length-3 training tuple by the mean future window that followed it. At inference, it looks up the discretized suffix $u^{*}$ of the input window, using a nearest-neighbor fallback when needed, and fuses the retrieved payload with the backbone through a learned per-channel sigmoid gate. SPM combines established components \citep{rabanser2020discretization,lin2003sax,ansari2024chronos,dragon2025,tsmemory2026}, adding under 5\% parameters and a 7\% median step-time overhead. It is designed as a diagnostic rather than a proposed forecasting method: discrete keys, per-channel independence, and a disengageable gate keep the regime measurement interpretable. The gate equation, state bound, early-stopping account, and graph-diffusion variant appear in the technical appendix.

\section{The Regime on Real Benchmarks}

\begin{table*}[htbp]
\centering
\tabstyle
\setlength{\tabcolsep}{7.5pt}
\begin{tabular}{@{}lrrrrrrr@{}}
\toprule
\rowcolor{kilmist} & \multicolumn{2}{@{}c}{\hdr{Learning-free tiles ($\pm$95\% CI)}} & \multicolumn{4}{c}{\hdr{Learned plug-ins}} & \multicolumn{1}{c@{}}{\hdr{Zero-shot}} \\
\cmidrule(lr){2-3} \cmidrule(lr){4-7} \cmidrule(l){8-8}
\rowcolor{kilmist}  \hdr{Dataset} & \hdr{Period tile} & \hdr{Best-lag tile} & \hdr{FAN} & \hdr{GTR} & \hdr{SPM} & \hdr{RAFT} & \hdr{Chronos} \\
\midrule
\multicolumn{8}{@{}l}{\emph{Concentrated period, shared across channels (Table~\ref{tab:spectra}); the tile helps in aggregate:}} \\
ETTm1   & $\textcolor{helpblue}{-\mathbf{44.3} \pm 2.9}$ & $\textcolor{helpblue}{-37.4 \pm 3.5}$ & $\textcolor{helpblue}{-16.7}$ & $\textcolor{helpblue}{-17.2}$ & $\textcolor{helpblue}{-\phantom{0}5.8}$ & $\textcolor{helpblue}{-21.2}$ & $\textcolor{hurtrust}{+35.7}$ \\
Traffic & $\textcolor{helpblue}{-36.7 \pm 2.5}$ & $\textcolor{helpblue}{-34.1 \pm 2.5}$ & $\textcolor{helpblue}{-11.2}$ & $\textcolor{hurtrust}{+19.5}$ & $\textcolor{helpblue}{-\phantom{0}1.0}$ & $\textcolor{helpblue}{\mathbf{-39.9}}$ & $\textcolor{hurtrust}{+49.8}$ \\
ECL     & $\textcolor{helpblue}{-32.5 \pm 2.8}$ & $\textcolor{helpblue}{-\mathbf{35.3} \pm 2.8}$ & $\textcolor{helpblue}{-\phantom{0}7.3}$ & $\textcolor{hurtrust}{+20.7}$ & $\textcolor{helpblue}{-\phantom{0}2.3}$ & $\textcolor{helpblue}{-31.9}$ & $\textcolor{hurtrust}{+38.5}$ \\
ETTh1   & $\textcolor{helpblue}{-\phantom{0}8.2 \pm 2.9}$ & $\textcolor{hurtrust}{+\phantom{0}3.9 \pm 2.7}$ & $\textcolor{hurtrust}{+\phantom{0}1.2}$ & $\textcolor{hurtrust}{+\phantom{0}8.4}$ & $\textcolor{helpblue}{-\phantom{0}0.5}$ & $\textcolor{helpblue}{\mathbf{-16.3}}$ & $\textcolor{hurtrust}{+22.2}$ \\
\midrule
\multicolumn{8}{@{}l}{\emph{Concentrated but per-channel periods; no single lag fits all channels:}} \\
Weather & $\textcolor{hurtrust}{+24.8 \pm 6.6}$ & $\textcolor{hurtrust}{+12.6 \pm 4.1}$ & $\textcolor{hurtrust}{+\phantom{0}1.2}$ & $\textcolor{hurtrust}{+13.1}$ & $\textcolor{helpblue}{\mathbf{-\phantom{0}2.9}}$ & $\textcolor{hurtrust}{+\phantom{0}8.8}$ & $\textcolor{hurtrust}{+\phantom{0}6.4}$ \\
\midrule
\multicolumn{8}{@{}l}{\emph{Weak or trend-dominated spectra; no method separably wins:}} \\
ETTh2   & $\textcolor{hurtrust}{+\phantom{0}4.6 \pm 3.6}$ & $\textcolor{hurtrust}{+\phantom{0}3.8 \pm 5.2}$ & $\textcolor{hurtrust}{+50.1}$ & $\textcolor{hurtrust}{+54.8}$ & $\textcolor{hurtrust}{+\phantom{0}4.8}$ & $\textcolor{hurtrust}{+\phantom{0}1.5}$ & $\textcolor{helpblue}{-\phantom{0}0.1}$ \\
ETTm2   & $\textcolor{hurtrust}{+\phantom{0}3.7 \pm 5.6}$ & $\textcolor{hurtrust}{+\phantom{0}4.6 \pm 4.6}$ & $\textcolor{hurtrust}{+17.9}$ & $\textcolor{hurtrust}{+54.6}$ & $\textcolor{hurtrust}{+\phantom{0}3.0}$ & $\textcolor{hurtrust}{+\phantom{0}0.4}$ & $\textcolor{hurtrust}{+\phantom{0}3.6}$ \\
\bottomrule
\end{tabular}
\normalsize
\caption{Learning-free tiles and learned plug-ins at $S{=}12$, reported as paired $\Delta$MSE\% relative to each trained backbone. Negative values indicate improvement. Results aggregate 24 backbone-horizon cells (six backbones and four horizons), each averaged over three seeds. Tile columns show 95\% Student-$t$ CIs. Mechanism CIs and Wilcoxon tests appear in the technical appendix, where the closest entries in the bottom block are at parity. SPM is our instrument. Chronos is zero-shot Chronos-Bolt evaluated on the same normalized windows and is not considered for bolding. All methods follow one protocol, with plug-in settings listed under \textit{Study Design}. Bold marks the best value in a row when its 95\% CI excludes zero. Blue indicates improvement and red degradation.}
\label{tab:copy}
\label{tab:paradigm}
\end{table*}

\textbf{\emph{Does a learning-free tile really beat trained deep models?}} Yes, on four of seven benchmarks (Table~\ref{tab:copy}). At $S{=}12$, the period tile reduces aggregate MSE by 44.3\% on ETTm1, 36.7\% on Traffic, and 32.5\% on ECL, winning every paired run for $H \in \{96,336\}$. It also reduces MSE by 8.2\% on ETTh1, winning 92\% of paired runs, with one backbone-level margin within noise. Against learned plug-ins, it beats the best on ETTm1, matches RAFT on ECL within the cell-level CIs, and trails RAFT by three points on Traffic. On ETTm1 at $H{=}96$, reducing $S$ from 96 to 12 raises mean backbone MSE from 0.332 to 0.834. The tile recovers most of this loss, reaching 0.427. The tile instead hurts ETTh2 ($+4.6\%$), ETTm2 ($+3.7\%$), and Weather ($+24.8\%$). Per-channel validation-selected lags do not reverse these failures. On Weather, they reduce the penalty to $+12.6\%$. A periodogram-lag variant performs worse still.

\textbf{\emph{Does retrieval benefit fade as the window grows?}} To isolate this effect, we use SPM rather than the tile. At each $S$, SPM and its baseline share the same backbone, so the difference measures only what retrieval adds. Figure~\ref{fig:ssweep} shows that SPM's benefit consistently fades as $S$ increases from 12 to 96. Every dataset that benefits at $S{=}12$ crosses to zero or worse within this range. The crossings loosely track $L_p$: ETTh1 and Traffic ($L_p{=}24$) cross between 12 and 24; ECL ($L_p{=}24$) and ETTm1 ($L_p{=}96$) between 24 and 48; and Weather ($L_p{=}144$) between 48 and 96. ETTh2 and ETTm2, the two low-concentration datasets, never benefit at any tested $S$. On ETTm2, shortening the window from 96 to 12 raises backbone MSE only from 0.181 to 0.228 at $H{=}96$. This leaves little for a cross-window mechanism to recover. ETTm1 and ETTm2 share the same $S$, $L$, and $H$, yet SPM helps only on ETTm1. The contrast shows that the period must be strong, not merely present.

The real-data sweep shows that the boundary moves with $L$, but not by a fixed scaling law. Crossings lie near $L$ on long-period datasets and within a small multiple of $L$ on daily ones, yet ECL and Traffic share $L_p{=}24$ and cross at different windows. Their tile sweep also uses a fixed 168-step lag. We therefore quantify the boundary using the synthetic cross, where lag and period coincide by construction. Against retrained backbones, the period tile shows the same attenuation and crosses zero on every dataset where it helps at $S{=}12$ (technical appendix). Published gains at $S{=}96$ are consistent with the map. Weather's 144-step period and the 168-step weekly cycles in ECL and Traffic still exceed the window, placing FAN and GTR near the boundary where small gains are expected. By contrast, RAFT uses validation-tuned lookbacks of 720 on nearly every benchmark, placing it on the covered side.

\begin{figure}[t]
\centering
\includegraphics[width=0.72\textwidth]{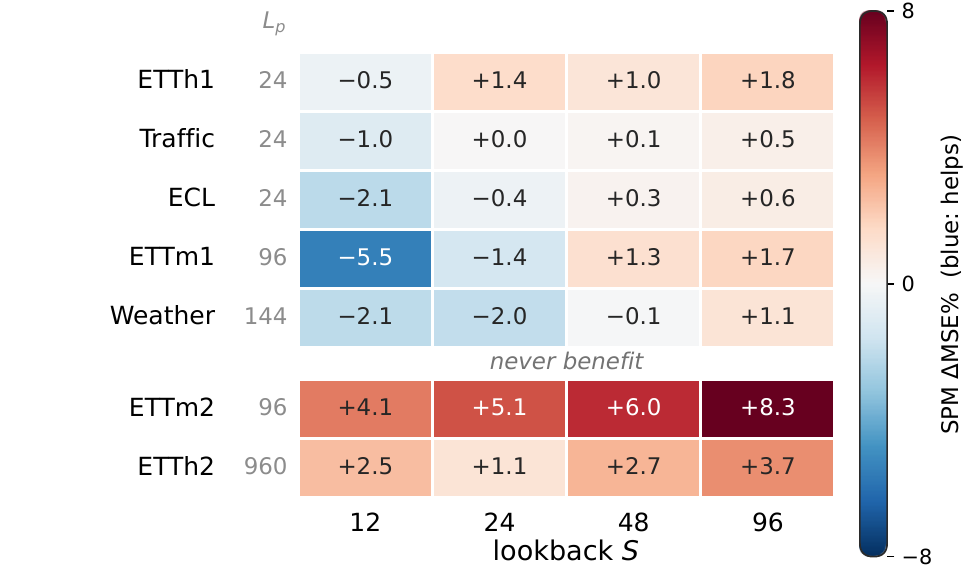}
\caption{Paired $\Delta$MSE\% of the SPM instrument as the window grows from 12 to 96, one row per dataset, benefiting rows ordered by $L_p$ and the never-benefiting pair grouped below, aggregated over twelve backbone-horizon cells (six backbones at $H{\in}\{96,336\}$, three seeds each); $L_p$ is the training-split median per-channel periodogram period (ETTh2: modal, trend-scale value). Blue (negative) means the instrument helps. Benefit attenuates and flips sign in the order of $L_p$, and the two low-concentration datasets (bottom) never benefit at any tested window.}
\label{fig:ssweep}
\end{figure}

\section{A Paired Synthetic Cross Isolates the Mechanism}

Real benchmarks confound period with sampling rate, channel count, and nonstationarity. We therefore sweep the three quantities in our claim independently: $H \in \{4,8,16,32\}$, $L \in \{4,8,16,32,64\}$, and 13 values of $S$ from 3 to 96, extended to 192 for the two largest periods. Using three seeds, a DLinear backbone, and the SPM probe gives 1{,}656 runs. Figure~\ref{fig:regime}(a) shows the $H{=}8$ slice for DLinear with SPM, our controlled backbone-probe pair. The same attenuation-and-crossing signature appears on five additional backbones in the synthetic cross and with FAN, GTR, and RAFT on real benchmarks, as shown under \textit{Robustness and Generalization}. Each series is a strictly periodic smooth waveform built from Fourier harmonics of the base period, with random phases and light observation noise (full specification in the code appendix). Smoothness is essential: the waveform varies between observed points, so accurate forecasting requires recovering its phase.

\textbf{\emph{Does the boundary track the period or the horizon?}} For each $(H,L)$, we smooth the $S$ grid with a three-point median and define the boundary as the first zero crossing after the point of deepest benefit. This prevents isolated fluctuations from being mistaken for the boundary. Because percentage changes become unstable near the noise floor, we report absolute errors in the code appendix. Alternative crossing definitions yield the same conclusion and are tabulated in the technical appendix.

Across all twenty $(H,L)$ pairs, the crossing follows $L$ more closely than $H$, correlating at $+0.71$ with $L$ and $-0.23$ with $H$ on DLinear. Within each horizon, the boundary occurs at the same or larger $S$ as $L$ increases, but it grows more slowly than the period. It lies near $6L$ at $L{=}4$ and between $L$ and $2L$ at $L{=}64$ (Figure~\ref{fig:mechanism}). By contrast, $H$ mainly controls the size of the benefit. At $L{=}64$, the gain deepens from $-9\%$ at $H{=}4$ to $-17\%$ at $H{=}32$, while the boundary remains statistically flat across horizons. Thus, $L$ determines where the regime ends, and $H$ determines how much retrieval helps.

\textbf{\emph{Is phase the quantity that matters?}} A paired phase-free control isolates this effect. It replaces the smooth periodic cycle with motifs of the same length that recur at random times, removing any phase that a short window could miss. Under the same analysis, the boundary correlation with $L$ falls to $-0.25$, all twenty crossings occur by $S{\le}16$, and in-regime gains shrink to $0.7\%$--$2.7\%$, compared with $7\%$--$19\%$ on the smooth cross (Figure~\ref{fig:mechanism}). Thus, removing periodic timing largely eliminates both the $L$-dependent boundary and the large retrieval gains. To formalize this intuition, treat the waveform as unknown and revealed only at the phases covered by the input window.

\begin{proposition}[Phase starvation]\label{prop:phase}
Let $y_t = \Phi(t \bmod L)$, the waveform $\Phi = (\Phi(0), \dots, \Phi(L{-}1))$ drawn from a prior with $\operatorname{Var}(\Phi(q) \mid \Phi(A)) \ge v > 0$ almost surely for every phase $q$ and every set $A$ of at most $S$ phases with $q \notin A$. Call a predictor \emph{window-only} if it is a measurable function of $(y_{t-S+1}, \dots, y_t)$ alone. If $S < L$, every window-only predictor incurs risk at least $v$ on every horizon step whose target phase the window does not cover, hence at least $v\,(1 - S/L)$ per step averaged over any horizon $H = mL$ (integer $m$). A predictor with access to the realized waveform, which the training record supplies once it spans one period, incurs none of this term; at $S \ge L$ the term vanishes, and the period tile attains zero waveform risk.
\end{proposition}

\noindent\emph{Proof sketch.} Squared-error risk is at least the conditional variance of the target given the window. The window is a function of the covered phase values and the alignment, and revealing the alignment only strengthens the bound, so the assumption gives the per-step floor. Exactly $L{-}S$ of every $L$ consecutive horizon steps land on uncovered phases. Observation noise adds to every term. The full proof and two priors satisfying the assumption, one exactly and one off degenerate configurations, are in the technical appendix.

The bound matches the geometry in Figure~\ref{fig:teaser}: window-only risk grows with the starved fraction $1-S/L$. The term $v$ captures waveform uncertainty that the observed phases do not resolve, for which spectral concentration serves as an empirical proxy. A trained backbone is not window-only because its weights encode information from the training record. Its failure to close the gap at $S \ll L$ is therefore an empirical result, not a consequence of the proposition.

\textbf{\emph{Does concentration cause the benefit, or only correlate with it?}} To isolate its effect, a third sweep holds the waveform and operating point fixed across two starved grids while varying only the share of periodic power. Retrieval benefit increases monotonically with concentration, from $0.4\%$ at ETTh2's level of 0.08 to $12\%$ at ECL's level of 0.60 ($\Delta$MSE\% against concentration: $r{=}{-}0.93$ over 60 paired runs; construction in the technical appendix).

\begin{figure}[t]
\centering
\includegraphics[width=0.590\textwidth]{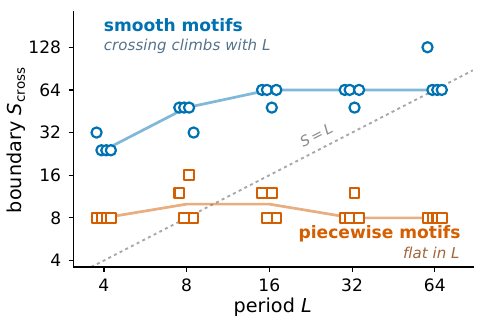}
\caption{Boundary location versus period for the two synthetic designs, all horizons overlaid, on log-log axes. Solid lines trace per-$L$ medians. With smooth, phase-informative waveforms (circles) the zero crossing climbs with the period $L$; with discrete-motif waveforms (squares), aperiodic and hence phase-free by construction, it does not. The smooth-cross grid extends to $S{=}192$ for the two largest periods, so no crossing is censored.}
\label{fig:mechanism}
\end{figure}

\section{Plug-in Families at the Operating Point}
The tile columns of Table~\ref{tab:copy} show that the starved operating point leaves useful information outside the window. The mechanism columns test whether different methods can recover that information. They compare four plug-in families and a zero-shot foundation model at the same operating point. Each method succeeds only when its design can carry the missing structure from the training record. Weather's middle row separates the two outer blocks because its periods are strong within channels but differ across them. The variation across channels favors a per-channel mechanism over a single-lag tile.

\paragraph{Parametric priors pay only where their assumption holds} FAN helps on ECL, Traffic, and ETTm1 ($-7.3\%$, $-11.2\%$, $-16.7\%$) and hurts elsewhere, up to $+50.1\%$ on ETTh2. Its frequency budget is capped by the window: a 12-point input has six positive-frequency rFFT bins, so a top-8 budget is not runnable at $S{=}12$, and FAN's published ETTm1 (top-11) and Traffic (top-30) budgets are likewise infeasible. GTR's learned cycle embedding helps on ETTm1 ($-17.2\%$) and fails by up to $+54.8\%$ elsewhere. Tuning its cycle per dataset leaves it a net failure on six of seven (technical appendix), so at $S{=}12$ the failure belongs to the mechanism, not the cycle.

\paragraph{Corpus retrieval pays at corpus prices} RAFT is the strongest learned mechanism on every strongly periodic benchmark, as the regime account expects, since it injects the very structure the window lacks. Its ETTh2 and ETTm2 results ($+1.5\%$, $+0.4\%$) are at parity, both intervals straddling zero. $S{=}12$ sits far from RAFT's published lookbacks, tuned over a 96-to-720 grid and 720 on nearly every benchmark, so its strength here is evidence about the mechanism class, not a re-evaluation of the method. The costs are corpus sized, with storage $O(NHD)$ in the training length $N$, and it fails where stored windows cannot resolve heterogeneous per-channel phases (Weather, $+8.8\%$).

\paragraph{A bounded instrument} SPM is the top scorer only on Weather, the one dataset where it is also the only learned mechanism that helps and where per-channel independence is the requirement. Its dataset-level failures stay within $+4.8\%$, against $+50.1\%$ for FAN and $+54.8\%$ for GTR. Instrumented runs attribute the bound to early stopping rather than movement of the per-channel gate, and it is loose: the worst single cell reaches $+31.5\%$ at the longest horizon (technical appendix).

\textbf{\emph{Does zero-shot pretraining escape the regime?}} No. Chronos-Bolt \citep{ansari2024chronos}, given the same 12-step normalized windows, trails trained backbones by $+22.2\%$ to $+49.8\%$ on every strongly periodic benchmark and stays within single digits only where the spectrum is weak or heterogeneous (ETTh2 $-0.1\%$, ETTm2 $+3.6\%$, Weather $+6.4\%$). A zero-shot model is window-only in the sense of Proposition~\ref{prop:phase}: its weights encode a prior, not this series' record. At $S \ll L$ the starved phases are therefore unavailable to it. A matched $S{=}96$ control agrees: the deficits collapse to parity on ECL and Traffic, where the window covers the period, and widen on ETTm1, whose 96-step period the window only just reaches, and on Weather, whose 144-step period it does not cover.

\textbf{\emph{Does mechanism sophistication matter inside the regime?}} Within the instrument, surprisingly little. Replacing the instrument's graph-diffusion retrieval with the exact tuple entry, across 7 datasets $\times$ 6 backbones $\times$ 2 horizons $\times$ 3 seeds, changes the result by at most $0.28\%$ (pooled $+0.07 \pm 0.10$, Wilcoxon over 252 runs, $p{=}0.88$), licensing SPM as a clean within-family probe. Yet Table~\ref{tab:paradigm} shows that crossing families buys 16 to 39 points on the periodic benchmarks at corpus prices.

\section{A Pre-Deployment Diagnostic}

To locate the regime before deployment, concentration and $S/L$ provide a coarse map, while two label-free statistics track the remaining failure modes. The channel-averaged training-split ADF $p$-value \citep{adf1979,said1984adf} tracks trend domination, as in ETTh2; the symbolic out-of-vocabulary rate (OOV), the worst per-channel fraction of deployment tuples absent from the training index, tracks memory staleness. Proposition~\ref{prop:risk} formalizes both. Weather needs no additional diagnostic because SPM's per-channel design handles its heterogeneity (Table~\ref{tab:paradigm}).

\begin{proposition}[Deployment risk of exact lookup]\label{prop:risk}
With targets bounded by $B$, write $R_{\mathrm{sym}}$ for the Bayes risk of predicting the future from the instrument's queried symbolic state $u^*$ under the test law, $D_{\mathrm{TV}}(u)$ for the train-test total-variation drift of the conditional future law at $u$, and $\delta$ for the out-of-vocabulary mass. Exact lookup at training-population payloads satisfies, per forecast coordinate,
\begin{equation}\label{eq:riskmain}
\mathbb{E}_{\mathrm{te}}\big[(\hat{y}_{\mathrm{retr}} - y)^2\big] \le R_{\mathrm{sym}} + 4B^2 \mathbb{E}_{\mathrm{te}}\big[D_{\mathrm{TV}}(u^*)\big] + 4B^2 \delta.
\end{equation}
\end{proposition}
ADF proxies the drift term, while OOV estimates $\delta$ directly. Because the bound does not compare $R_{\mathrm{sym}}$ with backbone risk, it motivates these diagnostics without proving the regime claim (technical appendix). Across 84 dataset-backbone-horizon cells ($7{\times}6{\times}2$; three seeds each), Spearman correlations between paired benefit and ADF and OOV are $0.37$ and $0.11$, respectively, rising to $0.62$ and $0.50$ on the five datasets run with all channels.

Table~\ref{tab:lodo} evaluates leave-one-dataset-out (LODO) prediction of the held-out dataset's per-cell benefit sign. A logistic rule using ADF and OOV reaches 0.76 accuracy, compared with 0.69 for the majority rule. Adding kurtosis lowers accuracy to 0.57, the same as a 22-feature catch22 model with XGBoost \citep{lubba2019catch22}. This supports the concern that generic descriptors overfit dataset identity \citep{akinci2026selection}. The pair's advantage comes entirely from the trend-dominated ETTh2 fold, while the two rules agree on the other six folds. The improvement is only 6 of 84 cells, and an exact McNemar test is not significant ($p{=}0.15$). Leave-one-out regression is also weak ($\rho{\approx}0.09$), so the diagnostic predicts whether to deploy rather than the size of the benefit.

The strongest test is data the diagnostic never saw: Exchange (ADF $p{=}0.36$, train-half OOV $0.80$) and ILI (weekly influenza counts, ADF $p{=}0.04$), withheld from every stage of development. Retrieval is useless on Exchange ($+18.8\%$) and pays on ILI through the period-reaching RAFT plug-in ($-6.8\%$). The frozen rule gets both dataset-level verdicts right. Per cell, scored on the instrument's sign, it reaches 0.62 (95\% CI [0.43, 0.79]), level with the novel data's own majority rate but far above the 0.38 of the development majority transferred as-is.

\begin{table}[htbp]
\centering
\tabstyle
\begin{tabular}{lr}
\toprule
\rowcolor{kilmist}  \hdr{Predictor (LODO, $n{=}84$ cells)} & \hdr{Sign accuracy} \\
\midrule
ADF $+$ OOV (2 features)            & $\mathbf{0.76}$ \\
ADF $+$ OOV $+$ kurtosis            & $0.57$ \\
catch22 $+$ XGBoost (22 features)   & $0.57$ \\
catch22 $+$ ADF $+$ OOV             & $0.57$ \\
Majority baseline                   & $0.69$ \\
\bottomrule
\end{tabular}
\normalsize
\caption{Leave-one-dataset-out prediction of the per-cell benefit sign, 84 cells. Bold: best accuracy.}
\label{tab:lodo}
\end{table}

\section{Robustness and Generalization}

Seven replications vary axes fixed by the main protocol and consistently support the coarse regime map, although the finer ADF$+$OOV rule is less stable. First, the tile's advantage survives backbone tuning. PatchTST's default patch length of 16 exceeds the $S{=}12$ window, raising the possibility that its loss reflects misconfiguration. We therefore retune the lookback-relevant hyperparameters of PatchTST and iTransformer on ETTm1, Traffic, and ECL. Tuning improves the backbones by 3--32\%, yet the tile still beats the oracle-best configuration on all six pairs by 12--54\% (technical appendix).

The boundary also transfers across backbones and retrieval mechanisms. On the synthetic cross with five additional backbones (7{,}404 runs), the crossing's correlation with $L$ lies in $[+0.64,+0.91]$ for five of six models. TimeMixer is the exception because its $L{=}4$ row never benefits: its all-cell correlation is $-0.09$, rising to $+0.89$ when restricted to cells where copy helps. On the real benchmarks, six of eight helpful plug-in-dataset pairs attenuate monotonically as $S$ grows across FAN, GTR, and RAFT (1{,}620 runs). On synthetic data, concentration has a causal effect on retrieval benefit. Within each of two fixed grids, increasing only the periodic-power share deepens the gain monotonically from ETTh2's concentration level to ECL's ($r{=}{-}0.93$ over 60 paired runs).

The result is also robust to protocol changes and new domains. Full-channel runs on ECL and Traffic reproduce the instrument's effect ($-2.4\%$ and $-1.1\%$), showing that the first-20-channel cap makes the tile comparison conservative. Moirai \citep{woo2024moirai} and Sundial \citep{liu2025sundial} likewise fail to close the gap under the same starved context. Beyond the original benchmarks, the frozen ADF$+$OOV rule predicts the per-cell benefit sign on Exchange and ILI with 0.62 accuracy, compared with 0.38 for the transferred in-distribution majority rule over 24 cells. On M5, concentration predicts where copy helps with AUC 0.97 across 9{,}109 daily retail series, from smooth aggregates to intermittent item-level counts. The finer ADF$+$OOV rule does not transfer to M5, reaching only 0.31, so the coarse regime axis transfers where the dataset-sensitive per-cell diagnostic does not.

\section{Discussion}

Our protocol fills a missing window axis in concurrent work: RAFT fixes the lookback, while \citet{ahuja2026retrieval} study the long-context end. Neither varies the window-to-period ratio or includes a control that can switch retrieval off. For practice, the resulting map is a guide rather than a guarantee. Estimate concentration and $S/L$ from the training split. Large gains are most likely when a strong period exceeds the window. Prefer per-channel mechanisms when periods differ across channels, and expect little benefit under trend domination or weak spectra.

\paragraph{Limitations} First, $S{=}12$ is an instrumented stress test, so its margins characterize the regime rather than a recommended setting. The same regime applies whenever the deployable window is shorter than a dominant cycle. Second, OOV uses evaluation-split inputs, but never labels. A variant computed between two training halves achieves the same 0.76 leave-one-dataset-out accuracy (technical appendix). Third, the phase-free arm of the synthetic cross is evaluated only on DLinear.

\section{Conclusion}

Retrieval plug-ins are neither uniformly helpful nor harmful. Their value depends on the window, period, and spectral concentration. When the window does not span a concentrated period, even a learning-free tile can beat trained backbones. Once it covers the period and longer cycles, no tested mechanism provides reliable gains. This boundary transfers across model families, and the concentration axis generalizes to M5. Our contribution is a regime map, a stratified evaluation protocol, a concentration screen, and two label-free diagnostics for predicting the sign of per-cell benefit.

\bibliography{refs}

@inproceedings{ye2024fan,
 author = {Ye, Weiwei and Deng, Songgaojun and Zou, Qiaosha and Gui, Ning},
 booktitle = {Advances in Neural Information Processing Systems},
 doi = {10.52202/079017-0985},
 editor = {A. Globerson and L. Mackey and D. Belgrave and A. Fan and U. Paquet and J. Tomczak and C. Zhang},
 pages = {31350--31379},
 publisher = {Curran Associates, Inc.},
 title = {Frequency Adaptive Normalization For Non-stationary Time Series Forecasting},
 url = {https://proceedings.neurips.cc/paper_files/paper/2024/file/37c6d0bc4d2917dcbea693b18504bd87-Paper-Conference.pdf},
 volume = {37},
 year = {2024}
}

@InProceedings{han2025raft,
  title = 	 {Retrieval Augmented Time Series Forecasting},
  author =       {Han, Sungwon and Lee, Seungeon and Cha, Meeyoung and Arik, Sercan O and Yoon, Jinsung},
  booktitle = 	 {Proceedings of the 42nd International Conference on Machine Learning},
  pages = 	 {21774--21797},
  year = 	 {2025},
  editor = 	 {Singh, Aarti and Fazel, Maryam and Hsu, Daniel and Lacoste-Julien, Simon and Berkenkamp, Felix and Maharaj, Tegan and Wagstaff, Kiri and Zhu, Jerry},
  volume = 	 {267},
  series = 	 {Proceedings of Machine Learning Research},
  month = 	 {13--19 Jul},
  publisher =    {PMLR},
  url = 	 {https://proceedings.mlr.press/v267/han25d.html},
}

@inproceedings{ratd2024,
 author = {Liu, Jingwei and Yang, Ling and Li, Hongyan and Hong, Shenda},
 booktitle = {Advances in Neural Information Processing Systems},
 doi = {10.52202/079017-0091},
 editor = {A. Globerson and L. Mackey and D. Belgrave and A. Fan and U. Paquet and J. Tomczak and C. Zhang},
 pages = {2766--2786},
 publisher = {Curran Associates, Inc.},
 title = {Retrieval-Augmented Diffusion Models for Time Series Forecasting},
 url = {https://proceedings.neurips.cc/paper_files/paper/2024/file/053ee34c0971568bfa5c773015c10502-Paper-Conference.pdf},
 volume = {37},
 year = {2024}
}

@inproceedings{
cao2026gtr,
title={Enhancing Multivariate Time Series Forecasting with Global Temporal Retrieval},
author={Fanpu Cao and Lu Dai and Jindong Han and Hui Xiong},
booktitle={The Fourteenth International Conference on Learning Representations},
year={2026},
url={https://openreview.net/forum?id=QUJBPSfyui}
}

@misc{butera2026context,
      title={Why Do Time Series Models Need Long Context Windows?}, 
      author={Luca Butera and Giovanni De Felice and Andrea Cini and Cesare Alippi},
      year={2026},
      eprint={2606.01999},
      archivePrefix={arXiv},
      primaryClass={cs.LG},
      url={https://arxiv.org/abs/2606.01999}, 
}

@misc{zhou2026saraf,
      title={Stationarity-Aware Retrieval-Augmented Time Series Forecasting}, 
      author={Shiqiao Zhou and Holger Schöner and Zipeng Wu and Edouard Fouché and IAG Wilson and Shuo Wang},
      year={2026},
      eprint={2606.04135},
      archivePrefix={arXiv},
      primaryClass={cs.LG},
      url={https://arxiv.org/abs/2606.04135}, 
}

@misc{ahuja2026retrieval,
      title={Retrieval Mechanisms Surpass Long-Context Scaling in Time Series Forecasting},
      author={Rishi Ahuja and Kumar Prateek and Simranjit Singh and Vijay Kumar},
      year={2026},
      eprint={2605.08217},
      archivePrefix={arXiv},
      primaryClass={cs.LG},
      url={https://arxiv.org/abs/2605.08217},
      note={Accepted at the 1st ICLR Workshop on Time Series in the Age of Large Models (TSALM)},
}

@inproceedings{
zhang2025parroting,
title={Context parroting: A simple but tough-to-beat baseline for foundation models in scientific machine learning},
author={Yuanzhao Zhang and William Gilpin},
booktitle={The Fourteenth International Conference on Learning Representations},
year={2026},
url={https://openreview.net/forum?id=EUAXc9Hlvm}
}

@misc{ride2026,
      title={Factorize to Generalize: Retrieval-Guided Invariant-Dynamic Decomposition for Time Series Forecasting}, 
      author={Jinjin Chi and Lei Feng and Lulu Zhang and Yongcheng Jing and Yiming Wang and Ximing Li and Jialie Shen and Leszek Rutkowski and Dacheng Tao},
      year={2026},
      eprint={2605.24911},
      archivePrefix={arXiv},
      primaryClass={cs.LG},
      url={https://arxiv.org/abs/2605.24911}, 
}

@misc{akinci2026selection,
      title={Why Model Selection Fails in Time Series Forecasting: An Empirical Study of Instability Across Data Regimes}, 
      author={Tahir Cetin Akinci and Alfredo A. Martinez-Morales},
      year={2026},
      eprint={2605.01608},
      archivePrefix={arXiv},
      primaryClass={eess.SP},
      url={https://arxiv.org/abs/2605.01608}, 
}

@inproceedings{zeng2023transformers,
author = {Zeng, Ailing and Chen, Muxi and Zhang, Lei and Xu, Qiang},
title = {Are transformers effective for time series forecasting?},
year = {2023},
isbn = {978-1-57735-880-0},
publisher = {AAAI Press},
url = {https://doi.org/10.1609/aaai.v37i9.26317},
doi = {10.1609/aaai.v37i9.26317},
booktitle = {Proceedings of the Thirty-Seventh AAAI Conference on Artificial Intelligence and Thirty-Fifth Conference on Innovative Applications of Artificial Intelligence and Thirteenth Symposium on Educational Advances in Artificial Intelligence},
articleno = {1248},
numpages = {8},
series = {AAAI'23/IAAI'23/EAAI'23}
}

@inproceedings{tan2024llm,
 author = {Tan, Mingtian and Merrill, Mike A. and Gupta, Vinayak and Althoff, Tim and Hartvigsen, Thomas},
 booktitle = {Advances in Neural Information Processing Systems},
 doi = {10.52202/079017-1922},
 editor = {A. Globerson and L. Mackey and D. Belgrave and A. Fan and U. Paquet and J. Tomczak and C. Zhang},
 pages = {60162--60191},
 publisher = {Curran Associates, Inc.},
 title = {Are Language Models Actually Useful for Time Series Forecasting?},
 url = {https://proceedings.neurips.cc/paper_files/paper/2024/file/6ed5bf446f59e2c6646d23058c86424b-Paper-Conference.pdf},
 volume = {37},
 year = {2024}
}

@misc{sun2022fredo,
      title={FreDo: Frequency Domain-based Long-Term Time Series Forecasting}, 
      author={Fan-Keng Sun and Duane S. Boning},
      year={2022},
      eprint={2205.12301},
      archivePrefix={arXiv},
      primaryClass={cs.LG},
      url={https://arxiv.org/abs/2205.12301}, 
}

@misc{gifteval2024,
      title={GIFT-Eval: A Benchmark For General Time Series Forecasting Model Evaluation}, 
      author={Taha Aksu and Gerald Woo and Juncheng Liu and Xu Liu and Chenghao Liu and Silvio Savarese and Caiming Xiong and Doyen Sahoo},
      year={2024},
      eprint={2410.10393},
      archivePrefix={arXiv},
      primaryClass={cs.LG},
      url={https://arxiv.org/abs/2410.10393}, 
}

@inproceedings{shi2024scaling,
 author = {Shi, Jingzhe and Ma, Qinwei and Ma, Huan and Li, Lei},
 booktitle = {Advances in Neural Information Processing Systems},
 doi = {10.52202/079017-2650},
 editor = {A. Globerson and L. Mackey and D. Belgrave and A. Fan and U. Paquet and J. Tomczak and C. Zhang},
 pages = {83314--83344},
 publisher = {Curran Associates, Inc.},
 title = {Scaling Law for Time Series Forecasting},
 url = {https://proceedings.neurips.cc/paper_files/paper/2024/file/97c2f0fac182353062d304d0322ae285-Paper-Conference.pdf},
 volume = {37},
 year = {2024}
}

@InProceedings{lookbackbias2025,
author="Abdelmalak, Ibram
and Madhusudhanan, Kiran
and Choi, Jungmin
and Kl{\"o}tergens, Christian
and Yalavarthi, Vijaya Krishna
and Stubbemann, Maximilian
and Schmidt-Thieme, Lars",
editor="Wong, Raymond Chi-Wing
and Tong, Hanghang
and Lu, Hua
and Kwok, James
and Salim, Flora
and Song, Yuanfeng
and Yiu, Man Lung",
title="Channel Dependence, Limited Lookback Windows, and the Simplicity of Datasets: How Biased Is Time Series Forecasting?",
booktitle="Advances in Knowledge Discovery and Data Mining",
year="2026",
publisher="Springer Nature Singapore",
address="Singapore",
pages="585--597",
isbn="978-981-92-1462-4"
}

@inproceedings{cyclenet2024,
 author = {Lin, Shengsheng and Lin, Weiwei and Hu, Xinyi and Wu, Wentai and Mo, Ruichao and Zhong, Haocheng},
 booktitle = {Advances in Neural Information Processing Systems},
 doi = {10.52202/079017-3373},
 editor = {A. Globerson and L. Mackey and D. Belgrave and A. Fan and U. Paquet and J. Tomczak and C. Zhang},
 pages = {106315--106345},
 publisher = {Curran Associates, Inc.},
 title = {CycleNet: Enhancing Time Series Forecasting through Modeling Periodic Patterns},
 url = {https://proceedings.neurips.cc/paper_files/paper/2024/file/bfe7998398779dde03cad7a73b1f81b6-Paper-Conference.pdf},
 volume = {37},
 year = {2024}
}

@misc{mfrs2025,
      title={MFRS: A Multi-Frequency Reference Series Approach to Scalable and Accurate Time-Series Forecasting}, 
      author={Liang Yu and Lai Tu and Xiang Bai},
      year={2025},
      eprint={2503.08328},
      archivePrefix={arXiv},
      primaryClass={cs.LG},
      url={https://arxiv.org/abs/2503.08328}, 
}

@article{
ansari2024chronos,
title={Chronos: Learning the Language of Time Series},
author={Abdul Fatir Ansari and Lorenzo Stella and Ali Caner Turkmen and Xiyuan Zhang and Pedro Mercado and Huibin Shen and Oleksandr Shchur and Syama Sundar Rangapuram and Sebastian Pineda Arango and Shubham Kapoor and Jasper Zschiegner and Danielle C. Maddix and Hao Wang and Michael W. Mahoney and Kari Torkkola and Andrew Gordon Wilson and Michael Bohlke-Schneider and Bernie Wang},
journal={Transactions on Machine Learning Research},
issn={2835-8856},
year={2024},
url={https://openreview.net/forum?id=gerNCVqqtR}
}

@InProceedings{woo2024moirai,
  title = 	 {Unified Training of Universal Time Series Forecasting Transformers},
  author =       {Woo, Gerald and Liu, Chenghao and Kumar, Akshat and Xiong, Caiming and Savarese, Silvio and Sahoo, Doyen},
  booktitle = 	 {Proceedings of the 41st International Conference on Machine Learning},
  pages = 	 {53140--53164},
  year = 	 {2024},
  editor = 	 {Salakhutdinov, Ruslan and Kolter, Zico and Heller, Katherine and Weller, Adrian and Oliver, Nuria and Scarlett, Jonathan and Berkenkamp, Felix},
  volume = 	 {235},
  series = 	 {Proceedings of Machine Learning Research},
  month = 	 {21--27 Jul},
  publisher =    {PMLR},
  url = 	 {https://proceedings.mlr.press/v235/woo24a.html},
}

@article{montero2020fforma,
title = {FFORMA: Feature-based forecast model averaging},
journal = {International Journal of Forecasting},
volume = {36},
number = {1},
pages = {86-92},
year = {2020},
note = {M4 Competition},
issn = {0169-2070},
doi = {https://doi.org/10.1016/j.ijforecast.2019.02.011},
url = {https://www.sciencedirect.com/science/article/pii/S0169207019300895},
author = {Pablo Montero-Manso and George Athanasopoulos and Rob J. Hyndman and Thiyanga S. Talagala},
}

@article{talagala2023fforms,
author = {Talagala, Thiyanga S. and Hyndman, Rob J. and Athanasopoulos, George},
title = {Meta-learning how to forecast time series},
journal = {Journal of Forecasting},
volume = {42},
number = {6},
pages = {1476-1501},
doi = {https://doi.org/10.1002/for.2963},
url = {https://onlinelibrary.wiley.com/doi/abs/10.1002/for.2963},
eprint = {https://onlinelibrary.wiley.com/doi/pdf/10.1002/for.2963},
year = {2023}
}

@article{talagala2022fformpp,
title = {FFORMPP: Feature-based forecast model performance prediction},
journal = {International Journal of Forecasting},
volume = {38},
number = {3},
pages = {920-943},
year = {2022},
issn = {0169-2070},
doi = {https://doi.org/10.1016/j.ijforecast.2021.07.002},
url = {https://www.sciencedirect.com/science/article/pii/S0169207021001138},
author = {Thiyanga S. Talagala and Feng Li and Yanfei Kang},
}

@inproceedings{abdallah2025autoforecast,
author = {Abdallah, Mustafa and Rossi, Ryan and Mahadik, Kanak and Kim, Sungchul and Zhao, Handong and Bagchi, Saurabh},
title = {AutoForecast: Automatic Time-Series Forecasting Model Selection},
year = {2022},
isbn = {9781450392365},
publisher = {Association for Computing Machinery},
address = {New York, NY, USA},
url = {https://doi.org/10.1145/3511808.3557241},
doi = {10.1145/3511808.3557241},
booktitle = {Proceedings of the 31st ACM International Conference on Information \& Knowledge Management},
pages = {5–14},
numpages = {10},
location = {Atlanta, GA, USA},
series = {CIKM '22}
}

@Article{lubba2019catch22,
author={Lubba, Carl H.
and Sethi, Sarab S.
and Knaute, Philip
and Schultz, Simon R.
and Fulcher, Ben D.
and Jones, Nick S.},
title={catch22: CAnonical Time-series CHaracteristics},
journal={Data Mining and Knowledge Discovery},
year={2019},
month={Nov},
day={01},
volume={33},
number={6},
pages={1821-1852},
issn={1573-756X},
doi={10.1007/s10618-019-00647-x},
url={https://doi.org/10.1007/s10618-019-00647-x}
}

@misc{rabanser2020discretization,
      title={The Effectiveness of Discretization in Forecasting: An Empirical Study on Neural Time Series Models}, 
      author={Stephan Rabanser and Tim Januschowski and Valentin Flunkert and David Salinas and Jan Gasthaus},
      year={2020},
      eprint={2005.10111},
      archivePrefix={arXiv},
      primaryClass={cs.LG},
      url={https://arxiv.org/abs/2005.10111}, 
}

@inproceedings{lin2003sax,
author = {Lin, Jessica and Keogh, Eamonn and Lonardi, Stefano and Chiu, Bill},
title = {A symbolic representation of time series, with implications for streaming algorithms},
year = {2003},
isbn = {9781450374224},
publisher = {Association for Computing Machinery},
address = {New York, NY, USA},
url = {https://doi.org/10.1145/882082.882086},
doi = {10.1145/882082.882086},
booktitle = {Proceedings of the 8th ACM SIGMOD Workshop on Research Issues in Data Mining and Knowledge Discovery},
pages = {2–11},
numpages = {10},
location = {San Diego, California},
series = {DMKD '03}
}

@misc{dragon2025,
      title={Multivariate de Bruijn Graphs: A Symbolic Graph Framework for Time Series Forecasting}, 
      author={Mert Onur Cakiroglu and Idil Bilge Altun and Mehmet Dalkilic and Elham Buxton and Hasan Kurban},
      year={2025},
      eprint={2505.22768},
      archivePrefix={arXiv},
      primaryClass={cs.LG},
      url={https://arxiv.org/abs/2505.22768}, 
}

@misc{tsmemory2026,
      title={TS-Memory: Plug-and-Play Memory for Time Series Foundation Models}, 
      author={Sisuo Lyu and Siru Zhong and Tiegang Chen and Weilin Ruan and Qingxiang Liu and Taiqiang Lv and Qingsong Wen and Raymond Chi-Wing Wong and Yuxuan Liang},
      year={2026},
      eprint={2602.11550},
      archivePrefix={arXiv},
      primaryClass={cs.LG},
      url={https://arxiv.org/abs/2602.11550}, 
}

@inproceedings{
nie2023patchtst,
title={A Time Series is Worth 64 Words:  Long-term Forecasting with Transformers},
author={Yuqi Nie and Nam H Nguyen and Phanwadee Sinthong and Jayant Kalagnanam},
booktitle={The Eleventh International Conference on Learning Representations },
year={2023},
url={https://openreview.net/forum?id=Jbdc0vTOcol}
}

@inproceedings{liu2024itransformer,
 author = {Liu, Yong and Hu, Tengge and Zhang, Haoran and Wu, Haixu and Wang, Shiyu and Ma, Lintao and Long, Mingsheng},
 booktitle = {International Conference on Learning Representations},
 editor = {B. Kim and Y. Yue and S. Chaudhuri and K. Fragkiadaki and M. Khan and Y. Sun},
 pages = {11116--11140},
 title = {iTransformer: Inverted Transformers Are Effective for Time Series Forecasting},
 url = {https://proceedings.iclr.cc/paper_files/paper/2024/file/2ea18fdc667e0ef2ad82b2b4d65147ad-Paper-Conference.pdf},
 volume = {2024},
 year = {2024}
}

@inproceedings{
wu2023timesnet,
title={TimesNet: Temporal 2D-Variation Modeling for General Time Series Analysis},
author={Haixu Wu and Tengge Hu and Yong Liu and Hang Zhou and Jianmin Wang and Mingsheng Long},
booktitle={The Eleventh International Conference on Learning Representations },
year={2023},
url={https://openreview.net/forum?id=ju_Uqw384Oq}
}

@inproceedings{wang2024timemixer,
 author = {Wang, Shiyu and Wu, Haixu and Shi, Xiaoming and Hu, Tengge and Luo, Huakun and Ma, Lintao and Zhang, James and ZHOU, JUN},
 booktitle = {International Conference on Learning Representations},
 editor = {B. Kim and Y. Yue and S. Chaudhuri and K. Fragkiadaki and M. Khan and Y. Sun},
 pages = {38626--38652},
 title = {TimeMixer: Decomposable Multiscale Mixing for Time Series Forecasting},
 url = {https://proceedings.iclr.cc/paper_files/paper/2024/file/a7ac8a21e5a27e7ab31a5f42a0117bdb-Paper-Conference.pdf},
 volume = {2024},
 year = {2024}
}

@article{said1984adf,
author = {Said E. Said and David A. Dickey},
title = {Testing for Unit Roots in Autoregressive-Moving Average Models of Unknown Order},
journal = {Biometrika},
volume = {71},
number = {3},
pages = {599--607},
year = {1984},
}

@article{adf1979,
author = {David A. Dickey and Wayne A. Fuller},
title = {Distribution of the Estimators for Autoregressive Time Series with a Unit Root},
journal = {Journal of the American Statistical Association},
volume = {74},
number = {366a},
pages = {427--431},
year = {1979},
publisher = {Taylor \& Francis},
doi = {10.1080/01621459.1979.10482531},
URL = {https://doi.org/10.1080/01621459.1979.10482531},
eprint = {https://doi.org/10.1080/01621459.1979.10482531}
}

@book{hyndman2018forecasting,
  title={Forecasting: principles and practice},
  author={Hyndman, Rob J and Athanasopoulos, George},
  year={2018},
  publisher={OTexts}
}

@inproceedings{tsrag2025,
 author = {Ning, Kanghui and Pan, Zijie and Liu, Yu and Jiang, Yushan and Zhang, James and Rasul, Kashif and Schneider, Anderson and Ma, Lintao and Nevmyvaka, Yuriy and Song, Dongjin},
 booktitle = {Advances in Neural Information Processing Systems},
 editor = {D. Belgrave and C. Zhang and H. Lin and R. Pascanu and P. Koniusz and M. Ghassemi and N. Chen},
 pages = {163170--163199},
 publisher = {Curran Associates, Inc.},
 title = {TS-RAG: Retrieval-Augmented Generation based Time Series Foundation Models are Stronger Zero-Shot Forecaster},
 url = {https://proceedings.neurips.cc/paper_files/paper/2025/file/eed25c037bc08afcbefab6f7a6b700e0-Paper-Conference.pdf},
 volume = {38},
 year = {2025}
}

@inproceedings{wang2024timexer,
 author = {Wang, Yuxuan and Wu, Haixu and Dong, Jiaxiang and Qin, Guo and Zhang, Haoran and Liu, Yong and Qiu, Yunzhong and Wang, Jianmin and Long, Mingsheng},
 booktitle = {Advances in Neural Information Processing Systems},
 doi = {10.52202/079017-0015},
 editor = {A. Globerson and L. Mackey and D. Belgrave and A. Fan and U. Paquet and J. Tomczak and C. Zhang},
 pages = {469--498},
 publisher = {Curran Associates, Inc.},
 title = {TimeXer: Empowering Transformers for Time Series Forecasting with Exogenous Variables},
 url = {https://proceedings.neurips.cc/paper_files/paper/2024/file/0113ef4642264adc2e6924a3cbbdf532-Paper-Conference.pdf},
 volume = {37},
 year = {2024}
}

@inproceedings{du2026pfrp,
  author    = {Du, Dazhao and Han, Tao and Guo, Song},
  title     = {Predicting the Future by Retrieving the Past},
  booktitle = {Proceedings of the AAAI Conference on Artificial Intelligence},
  volume    = {40},
  pages     = {20896--20904},
  year      = {2026},
  doi       = {10.1609/aaai.v40i25.39230}
}

@INPROCEEDINGS{kang2026craft,
  author={Kang, Junhyeok and Seo, Jun and Park, Soyeon and Han, Sangjun and Bae, Seohui and Choe, Hyeokjun and Lee, Soonyoung},
  booktitle={ICASSP 2026 - 2026 IEEE International Conference on Acoustics, Speech and Signal Processing (ICASSP)}, 
  title={Channel-Wise Retrieval for Multivariate Time Series Forecasting}, 
  year={2026},
  pages={1336--1340},
  doi={10.1109/ICASSP55912.2026.11463178}
}

@inproceedings{
qiao2026time,
title={It's {TIME}: Towards the Next Generation of Time Series Forecasting Benchmarks},
author={Zhongzheng Qiao and Sheng Pan and Anni Wang and Viktoriya Zhukova and Yong Liu and Xudong Jiang and Qingsong Wen and Mingsheng Long and Ming Jin and Chenghao Liu},
booktitle={Forty-third International Conference on Machine Learning},
year={2026},
url={https://openreview.net/forum?id=79TgfXHbsK}
}

@InProceedings{das2024timesfm,
  title = 	 {A decoder-only foundation model for time-series forecasting},
  author =       {Das, Abhimanyu and Kong, Weihao and Sen, Rajat and Zhou, Yichen},
  booktitle = 	 {Proceedings of the 41st International Conference on Machine Learning},
  pages = 	 {10148--10167},
  year = 	 {2024},
  editor = 	 {Salakhutdinov, Ruslan and Kolter, Zico and Heller, Katherine and Weller, Adrian and Oliver, Nuria and Scarlett, Jonathan and Berkenkamp, Felix},
  volume = 	 {235},
  series = 	 {Proceedings of Machine Learning Research},
  month = 	 {21--27 Jul},
  publisher =    {PMLR},
  url = 	 {https://proceedings.mlr.press/v235/das24c.html},
}

@InProceedings{liu2025sundial,
  title = 	 {Sundial: A Family of Highly Capable Time Series Foundation Models},
  author =       {Liu, Yong and Qin, Guo and Shi, Zhiyuan and Chen, Zhi and Yang, Caiyin and Huang, Xiangdong and Wang, Jianmin and Long, Mingsheng},
  booktitle = 	 {Proceedings of the 42nd International Conference on Machine Learning},
  pages = 	 {39295--39317},
  year = 	 {2025},
  editor = 	 {Singh, Aarti and Fazel, Maryam and Hsu, Daniel and Lacoste-Julien, Simon and Berkenkamp, Felix and Maharaj, Tegan and Wagstaff, Kiri and Zhu, Jerry},
  volume = 	 {267},
  series = 	 {Proceedings of Machine Learning Research},
  month = 	 {13--19 Jul},
  publisher =    {PMLR},
  url = 	 {https://proceedings.mlr.press/v267/liu25be.html},
  }

\clearpage
\appendix
\setcounter{topnumber}{2}
\setcounter{bottomnumber}{1}
\setcounter{totalnumber}{2}
\renewcommand{\topfraction}{0.7}
\renewcommand{\bottomfraction}{0.5}
\renewcommand{\textfraction}{0.2}
\renewcommand{\floatpagefraction}{0.85}

\renewcommand{\thetable}{A\arabic{table}}\setcounter{table}{0}
\renewcommand{\thefigure}{A\arabic{figure}}\setcounter{figure}{0}
\renewcommand{\theequation}{A\arabic{equation}}\setcounter{equation}{0}

\section{Absolute Error Scale}

Table~\ref{tab:absolute} reports the absolute mean MSE behind the paired percentages of the main text: the backbone at the standard window ($S{=}96$), the backbone at the starved window ($S{=}12$), the backbone with the SPM instrument at $S{=}12$, and the causal period tile. Means are over 6 backbones $\times$ 3 seeds on z-score normalized data. Table~\ref{tab:absolute_mae} is the MAE companion.

\section{Per-Channel Periodogram Periods}

For each channel we take the periodogram peak of the training split (linear detrending, periods 2 to 1{,}024). Table~\ref{tab:periods} summarizes the per-channel distribution by dataset. The main-text spectral table pairs each dataset's dominant period with its median spectral concentration, the share of detrended periodogram power at the modal period. The two datasets where no copy or retrieval mechanism helps (ETTh2, ETTm2) are the two lowest by concentration, while Weather concentrates power but at heterogeneous periods, the distinct failure mode visible in the wide range below. Full per-channel values ship in the code appendix. Restricting the search to periods at or below 504 does not change the medians except on trend-dominated ETTh2, whose median drops to the daily 24 once trend-scale peaks are excluded. It does not rescue the per-channel periodogram tile reported in the main text.

\begin{table}[htbp]
\centering
\tabstyle
\setlength{\tabcolsep}{6pt}
\begin{tabular}{lrll}
\toprule
\rowcolor{kilmist}  \hdr{Dataset} & \hdr{Median $L_p$} & \hdr{Range} & \hdr{Character} \\
\midrule
ETTh1   & 24  & 12 to 864   & daily \\
ETTh2   & 960 & 24 to 960   & trend scale \\
ETTm1   & 96  & 48 to 96    & daily (15-min) \\
ETTm2   & 96  & 96 to 886   & daily, weak power \\
ECL     & 24  & 12 to 1{,}023 & daily (mode 24) \\
Traffic & 24  & all at 24   & daily, uniform \\
Weather & 144 & 144 to 922  & heterogeneous \\
\bottomrule
\end{tabular}
\normalsize
\caption{Per-channel periodogram period $L_p$ by dataset: the median (ETTh2: modal, trend-scale value), the range across channels, and the qualitative character. Traffic's 862 channels all peak at 24; Weather and the trend-dominated ETTh2 spread widely.}
\label{tab:periods}
\end{table}

\begin{table}[htbp]
\centering
\tabstyle
\setlength{\tabcolsep}{1.6pt}
\begin{tabular}{
@{}lr rrrr
lr rrrr@{}
}
\toprule
\rowcolor{kilmist}  \multicolumn{6}{@{}c}{\hdr{\textit{ETT family}}} & \multicolumn{6}{c@{}}{\hdr{\textit{Other datasets}}} \\
\cmidrule(r{0.5em}){1-6} \cmidrule(l{0.5em}){7-12}
\rowcolor{kilmist}  \hdr{Dataset} & \hdr{$H$} & \hdr{Base@96} & \hdr{Base@12} & \hdr{SPM@12} & \hdr{Tile} & \hdr{Dataset} & \hdr{$H$} & \hdr{Base@96} & \hdr{Base@12} & \hdr{SPM@12} & \hdr{Tile} \\
\midrule
ETTh1 & 96 & 0.391 & 0.601 & 0.601 & 0.513 & Weather & 96 & 0.173 & 0.225 & 0.223 & 0.318 \\
 & 336 & 0.488 & 0.681 & 0.674 & 0.651 &  & 336 & 0.275 & 0.326 & 0.315 & 0.383 \\
\midrule
ETTh2 & 96 & 0.309 & 0.359 & 0.364 & 0.391 & ECL & 96 & 0.265 & 0.529 & 0.518 & 0.352 \\
 & 336 & 0.461 & 0.519 & 0.534 & 0.532 &  & 336 & 0.322 & 0.574 & 0.559 & 0.372 \\
\midrule
ETTm1 & 96 & 0.332 & 0.834 & 0.789 & 0.427 & Traffic & 96 & 0.434 & 0.844 & 0.832 & 0.519 \\
 & 336 & 0.413 & 0.912 & 0.857 & 0.500 &  & 336 & 0.472 & 0.865 & 0.859 & 0.542 \\
\midrule
ETTm2 & 96 & 0.181 & 0.228 & 0.244 & 0.264 &  &  &  &  &  &  \\
 & 336 & 0.322 & 0.388 & 0.395 & 0.377 &  &  &  &  &  &  \\
\bottomrule
\end{tabular}
\normalsize
\caption{Absolute mean test MSE. Base@96 and Base@12 are the plain backbones at the two windows; SPM@12 adds the instrument; Tile is the causal period copy (deterministic). ECL and Traffic use the first 20 channels in every column, including Base@96.}
\label{tab:absolute}
\end{table}

\begin{table}[htbp]
\centering
\tabstyle
\setlength{\tabcolsep}{1.6pt}
\begin{tabular}{
@{}lr rrrr
lr rrrr@{}
}
\toprule
\rowcolor{kilmist}  \multicolumn{6}{@{}c}{\hdr{\textit{ETT family}}} & \multicolumn{6}{c@{}}{\hdr{\textit{Other datasets}}} \\
\cmidrule(r{0.5em}){1-6} \cmidrule(l{0.5em}){7-12}
\rowcolor{kilmist}  \hdr{Dataset} & \hdr{$H$} & \hdr{Base@96} & \hdr{Base@12} & \hdr{SPM@12} & \hdr{Tile} & \hdr{Dataset} & \hdr{$H$} & \hdr{Base@96} & \hdr{Base@12} & \hdr{SPM@12} & \hdr{Tile} \\
\midrule
ETTh1 & 96 & 0.407 & 0.500 & 0.500 & 0.434 & Weather & 96 & 0.221 & 0.260 & 0.274 & 0.289 \\
 & 336 & 0.462 & 0.549 & 0.547 & 0.501 &  & 336 & 0.303 & 0.335 & 0.343 & 0.331 \\
\midrule
ETTh2 & 96 & 0.359 & 0.387 & 0.396 & 0.380 & ECL & 96 & 0.334 & 0.484 & 0.481 & 0.357 \\
 & 336 & 0.457 & 0.482 & 0.502 & 0.466 &  & 336 & 0.368 & 0.507 & 0.503 & 0.371 \\
\midrule
ETTm1 & 96 & 0.367 & 0.560 & 0.550 & 0.388 & Traffic & 96 & 0.352 & 0.530 & 0.528 & 0.300 \\
 & 336 & 0.415 & 0.606 & 0.593 & 0.427 &  & 336 & 0.364 & 0.535 & 0.534 & 0.306 \\
\midrule
ETTm2 & 96 & 0.267 & 0.307 & 0.324 & 0.301 &  &  &  &  &  &  \\
 & 336 & 0.359 & 0.404 & 0.418 & 0.371 &  &  &  &  &  &  \\
\bottomrule
\end{tabular}
\normalsize
\caption{MAE companion to Table~\ref{tab:absolute}: absolute mean test MAE under the same protocol. MAE agrees with MSE on every in-regime win, and on Traffic the tile's MAE advantage exceeds its MSE advantage. The tile's small out-of-regime MSE losses reverse under MAE on ETTh2 at both horizons, ETTm2 at $H{=}96$, and Weather at $H{=}336$, so those failure margins are partly MSE-specific.}
\label{tab:absolute_mae}
\end{table}

\section{The Headline Tables by Horizon}

The tile and mechanism columns of the main-text regime table aggregate over the four standard horizons $H \in \{96, 192, 336, 720\}$. Tables~\ref{tab:tile_perH} and~\ref{tab:paradigm_perH} break both out per horizon, and the four-horizon means in the main text are the row averages. The period tile's benefit is largest at the short horizons and attenuates as $H$ grows, since a longer forecast window dilutes the fraction a single copied period recovers, yet every sign is stable across horizons except at the near-zero boundary. The mechanism ordering is likewise horizon-stable: RAFT leads on the periodic benchmarks at every horizon and SPM stays bounded, while FAN and GTR inflate on the trend-dominated datasets (ETTh2, ETTm2) as $H$ grows, which is where all four mechanisms' worst cells fall.
\begin{table}[htbp]
\centering
\tabstyle
\setlength{\tabcolsep}{6pt}
\begin{tabular}{@{}lrrrr lrrrr@{}}
\toprule
\rowcolor{kilmist} \multicolumn{5}{@{}c}{\hdr{\textit{ETT family}}} & \multicolumn{5}{c@{}}{\hdr{\textit{Other datasets}}} \\
\cmidrule(r{0.5em}){1-5} \cmidrule(l{0.5em}){6-10}
\rowcolor{kilmist} \hdr{Dataset} & \hdr{$H{=}96$} & \hdr{$H{=}192$} & \hdr{$H{=}336$} & \hdr{$H{=}720$} & \hdr{Dataset} & \hdr{$H{=}96$} & \hdr{$H{=}192$} & \hdr{$H{=}336$} & \hdr{$H{=}720$} \\
\midrule
ETTh1   & $-14.0$ & $-10.0$ & $-4.2$  & $-4.7$ & Weather & $+42.7$ & $+30.1$ & $+17.7$ & $+8.8$ \\
ETTh2   & $+9.6$  & $+5.8$  & $+3.2$  & $0.0$ & ECL     & $-32.6$ & $-33.7$ & $-34.7$ & $-29.0$ \\
ETTm1   & $-48.2$ & $-46.0$ & $-44.7$ & $-38.4$ & Traffic & $-37.8$ & $-35.7$ & $-36.8$ & $-36.5$ \\
ETTm2   & $+16.5$ & $+5.8$  & $-1.6$  & $-6.0$ & &&&& \\
\bottomrule
\end{tabular}
\normalsize
\caption{Period tile paired $\Delta$MSE\% versus the trained backbones (negative helps), per horizon, over the six backbones with three seeds each. The main-text regime table's tile entries are the row averages.}
\label{tab:tile_perH}
\end{table}

\begin{table}[htbp]
\centering
\tabstyle
\setlength{\tabcolsep}{2.4pt}
\begin{tabular}{
@{}lr rrrrr
lr rrrrr@{}
}
\toprule
\rowcolor{kilmist}  \multicolumn{7}{@{}c}{\hdr{\textit{ETT family}}} & \multicolumn{7}{c@{}}{\hdr{\textit{Other datasets}}} \\
\cmidrule(r{0.5em}){1-7} \cmidrule(l{0.5em}){8-14}
\rowcolor{kilmist}  \hdr{Dataset} & \hdr{$H$} & \hdr{FAN} & \hdr{GTR} & \hdr{SPM} & \hdr{RAFT} & \hdr{Chronos} & \hdr{Dataset} & \hdr{$H$} & \hdr{FAN} & \hdr{GTR} & \hdr{SPM} & \hdr{RAFT} & \hdr{Chronos} \\
\midrule
ETTh1 & 96 & $+1.0$ & $+11.2$ & $+0.2$ & $-18.3$ & $+29.9$ & Weather & 96 & $+11.4$ & $+25.8$ & $-0.5$ & $+15.8$ & $+5.6$ \\
 & 192 & $-0.3$ & $+5.2$ & $-0.9$ & $-17.7$ & $+23.8$ &  & 192 & $+1.1$ & $+10.5$ & $-2.5$ & $+8.3$ & $+7.0$ \\
 & 336 & $+0.3$ & $+7.6$ & $-1.2$ & $-17.3$ & $+19.6$ &  & 336 & $-3.0$ & $+14.6$ & $-3.7$ & $+6.3$ & $+6.9$ \\
 & 720 & $+3.9$ & $+9.8$ & $-0.1$ & $-11.8$ & $+15.5$ &  & 720 & $-4.7$ & $+1.4$ & $-4.8$ & $+4.9$ & $+6.1$ \\
\midrule
ETTh2 & 96 & $+27.0$ & $+27.9$ & $+1.4$ & $-4.4$ & $-0.3$ & ECL & 96 & $-3.8$ & $+25.3$ & $-1.7$ & $-31.9$ & $+42.1$ \\
 & 192 & $+35.4$ & $+38.3$ & $+1.5$ & $-1.5$ & $-0.2$ &  & 192 & $-5.4$ & $+21.1$ & $-1.9$ & $-33.1$ & $+40.0$ \\
 & 336 & $+47.8$ & $+39.4$ & $+3.5$ & $+2.2$ & $+0.8$ &  & 336 & $-7.2$ & $+24.7$ & $-2.5$ & $-32.2$ & $+39.0$ \\
 & 720 & $+90.3$ & $+113.6$ & $+12.9$ & $+9.9$ & $-0.7$ &  & 720 & $-12.7$ & $+11.6$ & $-3.1$ & $-30.5$ & $+32.9$ \\
\midrule
ETTm1 & 96 & $-16.6$ & $-18.2$ & $-4.9$ & $-22.6$ & $+37.6$ & Traffic & 96 & $-10.4$ & $+22.2$ & $-1.3$ & $-38.3$ & $+53.5$ \\
 & 192 & $-16.3$ & $-18.8$ & $-5.5$ & $-21.3$ & $+36.9$ &  & 192 & $-11.1$ & $+20.7$ & $-1.0$ & $-39.4$ & $+52.7$ \\
 & 336 & $-17.0$ & $-18.7$ & $-6.0$ & $-21.9$ & $+35.4$ &  & 336 & $-11.5$ & $+18.2$ & $-0.8$ & $-40.4$ & $+51.0$ \\
 & 720 & $-17.0$ & $-13.2$ & $-6.6$ & $-19.2$ & $+32.9$ &  & 720 & $-12.0$ & $+17.0$ & $-0.7$ & $-41.5$ & $+42.0$ \\
\midrule
ETTm2 & 96 & $+8.1$ & $+22.6$ & $+6.8$ & $-2.1$ & $+9.4$ &  &  &  &  &  &  \\
 & 192 & $+11.8$ & $+42.4$ & $+2.4$ & $-1.5$ & $+4.7$ &  &  &  &  &  &  \\
 & 336 & $+18.9$ & $+68.5$ & $+1.5$ & $+0.4$ & $+1.8$ &  &  &  &  &  &  \\
 & 720 & $+32.9$ & $+84.8$ & $+1.2$ & $+4.7$ & $-1.5$ &  &  &  &  &  &  \\
\bottomrule
\end{tabular}
\normalsize
\caption{Plug-in families and zero-shot Chronos-Bolt, paired $\Delta$MSE\% versus the trained backbone (negative helps), per horizon, over the six backbones with three seeds each. The main-text regime table's mechanism entries are the four-horizon means. Failures on the trend-dominated datasets grow with the horizon.}
\label{tab:paradigm_perH}
\end{table}

\section{Full-Channel Replication for ECL and Traffic}

The $S{=}12$ protocol caps ECL (321 channels) and Traffic (862 channels) to their first 20 channels in every arm. To verify that the cap favors no mechanism, we reran the base backbones and the SPM instrument on all channels at $S{=}12$ (six backbones, $H \in \{96, 336\}$, three seeds) under the identical pairing convention, and recomputed the deterministic period tile (global weekly lag, all channels) against the full-channel base. Table~\ref{tab:fullchannel} sets the capped headline beside the full-channel result. The instrument reproduces: SPM is $-2.4\%$ on full-channel ECL (against $-2.1\%$ capped) and $-1.1\%$ on full-channel Traffic (against $-1.1\%$), both still small, negative, and bounded away from zero, so the cap does not manufacture the near-zero instrument result. The period tile remains a very strong baseline. On Traffic it is essentially unchanged ($-37.5\%$ against $-37.4\%$), and on ECL it deepens to $-46.0\%$ from $-33.6\%$ because the first twenty ECL channels carry higher normalized error than the dataset average (full-channel base mean MSE 0.43 against the capped 0.53 at $H{=}96$), so the cap if anything understated the tile rather than favoring it. The full-channel Traffic SPM mean is over all twelve backbone-horizon cells, now with complete three-seed coverage, per the strict-cell convention used throughout. Means and per-cell deltas come from the same pairing pipeline that reproduces every capped headline, and the tile is the deterministic global-lag copy, validated against the frozen capped-tile rows to within $1.2\times10^{-7}$ before the full-channel value was taken.

\begin{table}[htbp]
\centering
\tabstyle
\setlength{\tabcolsep}{6pt}
\begin{tabular}{llrr}
\toprule
\rowcolor{kilmist}  \hdr{Dataset} & \hdr{Arm} & \hdr{Capped} & \hdr{Full-channel} \\
\midrule
ECL     & Period tile    & $-33.6$ & $-46.0$ \\
ECL     & SPM instrument & $-2.1$  & $-2.4$  \\
Traffic & Period tile    & $-37.4$ & $-37.5$ \\
Traffic & SPM instrument & $-1.1$  & $-1.1$  \\
\bottomrule
\end{tabular}
\normalsize
\caption{Paired $\Delta$MSE\% at $S{=}12$ under the fixed 20-channel cap versus all channels, under the same pairing convention (negative helps). The tile is backbone-independent (broadcast over the twelve base cells); SPM is per-backbone, seed-paired, mean over the backbone-horizon cell means. Full-channel SPM is over $n{=}12$ complete three-seed cells on each of ECL and Traffic.}
\label{tab:fullchannel}
\end{table}

\section{The Window Sweep Under Both Horizon Conventions}

The main-text sweep aggregates $H \in \{96, 336\}$. The $S{=}24$ and $S{=}48$ sweeps were also run at $H \in \{192, 720\}$. Including all four horizons shifts no crossing by more than one grid step and leaves every ordering statement unchanged.

\section{Synthetic Cross: Grids and Absolute Errors}

The code appendix ships the full per-$(H, L, S)$ grids of paired $\Delta$MSE\% for the smooth and motif-control crosses, the corresponding absolute MSE tables (which confirm the large-$S$ percentage blow-ups sit on a noise-floor denominator near $\sigma^2$), the per-$(H, L)$ depth table, and the boundary statistics under several definitions.

\paragraph{The boundary correlation is robust to how the crossing is defined} On the de-censored grid ($S$ up to 192 for $L \in \{32, 64\}$, all twenty $(H, L)$ pairs defined), three of the four crossing definitions give essentially the same answer: the median-smoothed first crossing, the raw crossing after the deepest point, and the median-smoothed crossing after the deepest point (the one reported in the main text) all place the period correlation near $+0.71$ and the horizon correlation near $-0.23$ (precisely, $+0.71/{-}0.23$, $+0.72/{-}0.20$, and $+0.71/{-}0.23$). Only the raw first crossing differs ($+0.29$ with $L$, $+0.09$ with $H$), because single-point noise blips before the benefit peaks register as spurious crossings. The half-depth boundary gives the same verdict. The out-of-regime extension cells are uniformly positive (harm of $+1$ to $+13\%$ at $S \in \{128, 192\}$), confirming the large-period crossings are genuine.

\paragraph{The boundary survives two generator controls} Both rerun the boundary statistics on the $H{=}8$ row (780 runs) and are reported as sensitivity, not headline correlations. Freezing the harmonic count at 2 instead of scaling it with $L$ leaves the boundary tracking the period (correlation $+0.78$ over five pairs, all crossings interior), so waveform complexity is not what drives it. Quadrupling the observation noise to $\sigma{=}0.2$ leaves the trend intact ($+0.70$), with in-regime depths of $-7$ to $-17\%$.

\paragraph{The signature replicates on five further backbones} Rerunning the cross on all six backbones (7{,}404 runs in all) reproduces the period-tracking signature on every one (Table~\ref{tab:crossbb}). The crossing-versus-$H$ correlation stays near zero throughout ($-0.23$ to $+0.18$). The size of the help region varies, from all twenty $(H, L)$ cells on DLinear to a narrow long-period band on PatchTST, but where copy helps the boundary tracks the period on every backbone. For the headline pair, Fisher-$z$ intervals ($n{=}20$ pairs) put the period correlation in $[+0.40, +0.88]$ and the horizon correlation in $[-0.61, +0.24]$, and a 20{,}000-draw pair bootstrap agrees ($[+0.56, +0.88]$ and $[-0.56, +0.29]$). The horizon interval containing zero is itself the claim: the horizon does not set the boundary. The generator, dispatchers, and the analysis script ship in the code appendix.

\begin{table}[htbp]
\centering
\tabstyle
\setlength{\tabcolsep}{6pt}
\begin{tabular}{lrl}
\toprule
\rowcolor{kilmist}  \hdr{Backbone} & \hdr{corr$(S^*, L)$} & \hdr{Help region} \\
\midrule
DLinear      & $+0.71$ & all twenty $(H, L)$ cells \\
TimeXer      & $+0.91$ & broad \\
TimeMixer    & $+0.89$ & help cells only \\
iTransformer & $+0.87$ & broad \\
TimesNet     & $+0.86$ & reduced grid \\
PatchTST     & $+0.64$ & narrow long-period band \\
\bottomrule
\end{tabular}
\normalsize
\caption{Correlation of the per-row zero crossing $S^*$ with the period $L$, by backbone. DLinear is the full-grid value; TimesNet is on its reduced $H \in \{8, 32\}$, $S \le 96$ grid (against a like-for-like DLinear $+0.67$). TimeMixer's all-pairs correlation reads $-0.09$ because its short-period $L{=}4$ cells never benefit. Over the cells where copy helps it is $+0.89$.}
\label{tab:crossbb}
\end{table}

\section{Concentration as the Causal Variable}

The main-text spectral table is observational: across the seven benchmarks, benefit at the starved operating point correlates with spectral concentration. A controlled sweep makes the dependence causal. We hold a single starved operating point fixed ($S{=}12{<}L{=}24$, $H{=}16$) and a single waveform fixed, and dial only the fraction of variance the periodic structure carries; the run seed varies the observation noise alone. The measured concentration, computed with the exact main-text metric (detrended periodogram peak share over periods in $[2, 1024]$ on the train split), is then a clean monotone function of that fraction. The SPM instrument's paired $\Delta$MSE\% against its DLinear backbone deepens monotonically as concentration rises (Figure~\ref{fig:concsweep}), from neutral at the ETTh2 concentration to a double-digit gain at the ECL concentration. Pooled over the sixty paired runs (both operating points, ten concentration settings, three seeds each) the benefit tracks measured concentration at Pearson $r={-}0.93$ (Spearman $-0.95$, $p{<}10^{-26}$). A second operating point, $S{=}16{<}L{=}32$ at $H{=}16$, gives $r={-}0.93$ over its own grid. The two real anchors reproduce from a generator that shares nothing with the benchmarks but the spectrum: at concentration 0.08, the ETTh2 regime, the benefit is $-0.4\%$; at 0.60, the ECL regime, it is $-12.2\%$. Generator, dispatcher, and the analysis script ship in the code appendix.

\begin{figure}[t]
\centering
\includegraphics[width=0.549\textwidth]{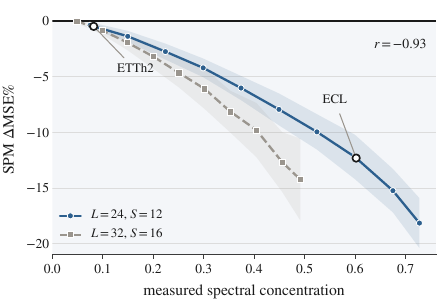}
\caption{Concentration causal sweep at two fixed starved operating points ($L{=}24$, $S{=}12$ and $L{=}32$, $S{=}16$, both at $H{=}16$). Dialing the periodic-power fraction raises measured spectral concentration (the main-text metric); the SPM instrument's paired $\Delta$MSE\% against its DLinear backbone deepens monotonically with it (each point is the mean over three noise seeds, shaded band $\pm 1$ standard deviation). The two open markers are the lowest and highest spectral concentration among the seven real benchmarks, ETTh2 (0.082) and ECL (0.602); the other five fall between them, so the synthetic sweep spans the full real range. At those two concentrations the synthetic benefit reproduces the measured one ($-0.4\%$ at ETTh2's, $-12.2\%$ at ECL's), from a generator that shares nothing with the benchmarks but the spectrum.}
\label{fig:concsweep}
\end{figure}

\section{Zero-Shot Arm Details}

Chronos-Bolt (base, 205M parameters) receives the identical z-score normalized 12-step contexts used by the trained arms, forecasts each channel independently, and rolls its native 64-step output autoregressively on the mean forecast to reach $H \in \{96, 192, 336, 720\}$. Inference uses every eighth test window. The subsampling audit computes the deterministic tile on both stride-1 and stride-8 window sets, and the absolute relative difference is at most 0.18\% across all datasets and horizons, with the maximum 0.174\% on ETTh1 at $H{=}336$. Its raw zero-shot test MSE is in Table~\ref{tab:zeroshot_mse}.

\begin{table}[htbp]
\centering
\tabstyle
\setlength{\tabcolsep}{6pt}
\begin{tabular}{@{}lrrrr lrrrr@{}}
\toprule
\rowcolor{kilmist} \multicolumn{5}{@{}c}{\hdr{\textit{ETT family}}} & \multicolumn{5}{c@{}}{\hdr{\textit{Other datasets}}} \\
\cmidrule(r{0.5em}){1-5} \cmidrule(l{0.5em}){6-10}
\rowcolor{kilmist} \hdr{Dataset} & \hdr{$H{=}96$} & \hdr{$H{=}192$} & \hdr{$H{=}336$} & \hdr{$H{=}720$} & \hdr{Dataset} & \hdr{$H{=}96$} & \hdr{$H{=}192$} & \hdr{$H{=}336$} & \hdr{$H{=}720$} \\
\midrule
ETTh1   & 0.775 & 0.800 & 0.814 & 0.797 & Weather & 0.235 & 0.282 & 0.348 & 0.433 \\
ETTh2   & 0.355 & 0.454 & 0.520 & 0.522 & ECL     & 0.743 & 0.751 & 0.792 & 0.861 \\
ETTm1   & 1.133 & 1.186 & 1.223 & 1.246 & Traffic & 1.282 & 1.249 & 1.296 & 1.283 \\
ETTm2   & 0.248 & 0.318 & 0.390 & 0.495 & &&&& \\
\bottomrule
\end{tabular}
\normalsize
\caption{Raw zero-shot test MSE of Chronos-Bolt (base) at $S{=}12$ across all four horizons, on z-score normalized data, for reference behind the paired percentages.}
\label{tab:zeroshot_mse}
\end{table}

To separate window coverage from the backbones' retraining, a matched control reruns the identical protocol at $S{=}96$ against backbones retrained at $S{=}96$, paired at the two standard horizons ($H \in \{96, 336\}$), the setting used for every analysis beyond the headline tables. The $S{=}12$ column here is therefore the two-horizon value and differs slightly from the four-horizon regime-table entry. Table~\ref{tab:zeroshot_match} shows the two operating points side by side. Where the longer window covers the period, the deficit collapses to parity: on ECL and Traffic (daily period) the gap falls from roughly $+40$ to $+50$ percent down to within its interval of zero, because the model's own MSE falls by over 60 percent once it can see a full day. ETTh1 halves its deficit without closing it, and ETTh2 sits at parity at both windows. On ETTm1, ETTm2, and Weather the gap instead widens sharply, because retrained backbones convert the longer window into large gains a frozen model cannot match, and on Weather the zero-shot error itself rises with the longer context. Two caveats. Horizons beyond the native 64 steps use autoregressive mean rolling, though the widenings are larger at $H{=}96$ than at $H{=}336$, so rolling is not their sole driver. And the comparison deliberately conflates window coverage with retraining, which is exactly the distinction the two operating points expose.

\begin{table}[htbp]
\centering
\tabstyle
\setlength{\tabcolsep}{5pt}
\begin{tabular}{@{}llrr llrr@{}}
\toprule
\rowcolor{kilmist} \multicolumn{4}{@{}c}{\hdr{\textit{ETT family}}} & \multicolumn{4}{c@{}}{\hdr{\textit{Other datasets}}} \\
\cmidrule(r{0.5em}){1-4} \cmidrule(l{0.5em}){5-8}
\rowcolor{kilmist} \hdr{Dataset} & \hdr{$L_p$} & \hdr{$S{=}12$} & \hdr{$S{=}96$} & \hdr{Dataset} & \hdr{$L_p$} & \hdr{$S{=}12$} & \hdr{$S{=}96$} \\
\midrule
ETTh1 & 24 & $+24.7$ & $+16.3$ & ECL & 24 & $+40.6$ & $+0.3$ \\
ETTh2 & 960 & $+0.2$ & $+4.0$ & Traffic & 24 & $+52.1$ & $+6.0$ \\
ETTm1 & 96 & $+36.3$ & $+176.1$ & Weather & 144 & $+6.3$ & $+80.9$ \\
ETTm2 & 96 & $+5.6$ & $+23.7$ &  &  &  &  \\
\bottomrule
\end{tabular}
\normalsize
\caption{Zero-shot Chronos-Bolt paired $\Delta$MSE\% against trained backbones at the two operating points (positive means the frozen model is worse). The deficit collapses where the longer window covers the period (ECL, Traffic) and widens where retrained backbones gain from it (ETTm1, ETTm2, Weather).}
\label{tab:zeroshot_match}
\end{table}

Two further checks confirm the conclusion does not depend on the model. Chronos-Bolt small, at $S{=}12$ under the identical protocol, is statistically indistinguishable from or better than the base model on all seven datasets (for example ETTh1 $+23.7$, ECL $+36.4$, Traffic $+37.9$, against $+24.7$, $+40.6$, $+52.1$ for base, with $\pm 5$ to $\pm 9$ CIs), so the gap is not an artifact of model scale. A second model family, Moirai (zero-shot, patch size fixed at 8), fails harder under the same $S{=}12$ context on all seven datasets, with paired $\Delta$MSE\% from $+69.2$ (ETTm2) to $+548.4$ (Traffic) and never below $+69\%$, uniformly worse than Chronos-Bolt. A third family, Sundial (\texttt{sundial-base-128m}, a 2025 flow-matching generative model), run at $S{=}12$ under the identical protocol, reproduces the Chronos pattern almost exactly: it trails the trained backbones by $+29.2\%$ (ETTh1), $+36.5\%$ (ETTm1), $+46.9\%$ (ECL), and $+53.5\%$ (Traffic) on the periodic benchmarks and stays within single digits on the weak-spectrum ones ($+1.2\%$ ETTh2, $+3.8\%$ ETTm2, $+7.5\%$ Weather). Its matched $S{=}96$ control reproduces the deficit collapse: the periodic gaps fall to $-2.3\%$ (ECL, from $+46.9$), $+12.9\%$ (Traffic, from $+53.5$), and $+5.6\%$ (ETTh1, from $+29.2$) where the longer window covers the daily period, and stay large or widen where it does not ($+40.4\%$ ETTm1, whose 96-step period the window only just reaches; $+12.9\%$ Weather, up from $+7.5$, whose 144-step period it still cannot cover). Traffic and Weather landing on the same $+12.9$ is coincidence, verified against the raw runs. The directions differ, a 40-point collapse against a widening.

\paragraph{Baseline currency} The zero-shot frontier moves quickly, so we add the 2025 flow-matching model Sundial \citep{liu2025sundial} as a third family (above). Further entrants such as TimesFM \citep{das2024timesfm} and mixture-of-experts successors would face the same constraint. The arm's conclusion does not hinge on which member is strongest: a zero-shot forecaster is window-only in the sense of Proposition~1 of the main text, its weights encoding a pretraining prior rather than the deployment series' record, so the starved phase values are unavailable to it at $S \ll L$ whatever the family or scale. A stronger model would have to overcome the same information constraint. The matched $S{=}96$ control isolates exactly this mechanism: the deficit collapses where the window covers the period and widens where it does not. The three families tested here (Chronos, Moirai, Sundial), across scales and architectures, show the pattern is not an artifact of one model.

\section{GTR Cycle Tuning}

The main-text regime table runs GTR with a fixed 168-step cycle. To check that GTR's failures reflect the mechanism rather than this choice, we tuned the cycle per dataset: a held-out probe on DLinear selected each dataset's better of a daily and a weekly cycle, and the winner was then run on all six backbones. Table~\ref{tab:gtrtune} reports the winners and the tuned result. The fixed 168-step choice was already optimal on the four hourly datasets. Even with the cycle tuned per dataset, GTR is a net failure on six of seven, never changes sign into a win, and is never the best mechanism in any row. Tuning the cycle does not rescue GTR, so its failures are a property of the cycle-prior mechanism at $S{=}12$, not of the fixed cycle length.

\begin{table}[htbp]
\centering
\tabstyle
\setlength{\tabcolsep}{6pt}
\begin{tabular}{@{}lrr lrr@{}}
\toprule
\rowcolor{kilmist} \multicolumn{3}{@{}c}{\hdr{\textit{ETT family}}} & \multicolumn{3}{c@{}}{\hdr{\textit{Other datasets}}} \\
\cmidrule(r{0.5em}){1-3} \cmidrule(l{0.5em}){4-6}
\rowcolor{kilmist} \hdr{Dataset} & \hdr{Tuned cycle} & \hdr{GTR $\Delta$MSE\%} & \hdr{Dataset} & \hdr{Tuned cycle} & \hdr{GTR $\Delta$MSE\%} \\
\midrule
ETTh1 & 168 & $+9.4$ & ECL & 168 & $+24.9$ \\
ETTh2 & 168 & $+33.7$ & Traffic & 168 & $+20.1$ \\
ETTm1 & 96 & $-16.9$ & Weather & 1008 & $+33.9$ \\
ETTm2 & 672 & $+55.0$ &  &  &  \\
\bottomrule
\end{tabular}
\normalsize
\caption{GTR with the cycle tuned per dataset (better of daily and weekly on a held-out DLinear probe), paired $\Delta$MSE\% over all six backbones (negative helps). The four hourly datasets keep the fixed 168-step cycle. GTR still fails on six of seven.}
\label{tab:gtrtune}
\end{table}

\section{Worst Cells}

Table~\ref{tab:worstcell} reports the single worst (backbone, horizon) cell for each mechanism at $S{=}12$, the tail behind the dataset-level means. The instrument's worst cell ($+31.5\%$) is an order of magnitude milder than GTR's ($+452.0\%$). All four mechanisms' worst cells fall on trend-dominated ETTh2 at the longest horizon, where a memory has nothing periodic to copy. Per-cell tables for every arm accompany the code appendix.

\begin{table}[htbp]
\centering
\tabstyle
\setlength{\tabcolsep}{6pt}
\begin{tabular}{llr}
\toprule
\rowcolor{kilmist}  \hdr{Mechanism} & \hdr{Worst cell (dataset, backbone, $H$)} & \hdr{$\Delta$MSE\%} \\
\midrule
SPM  & ETTh2, TimeMixer, 720 & $+31.5$  \\
FAN  & ETTh2, TimeMixer, 720 & $+105.4$ \\
RAFT & ETTh2, DLinear, 720   & $+58.7$  \\
GTR  & ETTh2, TimeMixer, 720 & $+452.0$ \\
\bottomrule
\end{tabular}
\normalsize
\caption{Worst single (backbone, horizon) cell per mechanism at $S{=}12$ (paired $\Delta$MSE\%, positive hurts). The instrument's second-worst cell is $+15.6\%$ (also ETTh2, at $H{=}720$). Away from ETTh2 and ETTm2 it stays within a few percent in every cell.}
\label{tab:worstcell}
\end{table}

\section{Measured Gate Values}

Fourteen instrumented DLinear runs (all seven datasets, both horizons, $S{=}12$) log the final per-channel gate values, 178 channel gates in total. The measurement contradicts the mechanism we designed for: the final $\sigma(g_d)$ values span 0.430 to 0.572, none falls below 0.1, and the gates barely leave their 0.5 initialization. Datasets where retrieval hurts show no practically meaningful gate suppression: the hurt-side median is in fact the higher one (0.4945 versus 0.4935), and while a one-sided rank test in the suppression direction reaches $p{=}0.043$, the disagreement between the rank and median views itself reflects how small the effect is (0.001 at the median, $n{=}178$). What bounds the out-of-regime failures is instead validation-based early stopping: on ETTh2 and ETTm2 training halts at best-validation epoch 1 to 3, against epoch 8 to 10 in regime, freezing the model near initialization. On ETTh2 the halted retrieval branch is still small (output magnitude ratio $|\hat{y}_{\mathrm{retr}}|/|\hat{y}_{\mathrm{base}}|$ of 0.08 at the restored checkpoint, against 0.16 to 0.45 in regime), which is the bounded-failure behavior the main text reports. The bound is loose, not a guarantee: in the instrumented ETTm2 worst cell ($+13.6\%$) the deployed gate averages 0.46, more than half open, with a branch ratio of 0.38. Gate values track per-channel periodicity only on Traffic (Spearman $\rho$ of 0.73 and 0.76 at the two horizons, $n{=}20$ channels each); pooled across the periodic datasets the correlation is absent ($\rho{=}-0.16$, $n{=}108$). The ETTh2/TimeMixer worst cell was not instrumented; its DLinear proxy restores the epoch-1 checkpoint with the gate at initialization (mean 0.499) and a branch ratio of 0.08.

\section{Uncertainty for the Mechanism Table}

Table~\ref{tab:paradigm_ci} reports $t$-based 95\% confidence intervals for every mechanism entry of the main-text regime table, computed over the same twenty-four backbone-horizon cell means as the means themselves, the convention used throughout the paper.

\begin{table*}[t]
\centering
\tabstyle
\setlength{\tabcolsep}{6pt}
\begin{tabular}{lrrrrr}
\toprule
\rowcolor{kilmist}  \hdr{Dataset} & \hdr{FAN} & \hdr{GTR} & \hdr{SPM} & \hdr{RAFT} & \hdr{Chronos} \\
\midrule
ETTh1   & $+1.2 \pm 2.0^{\dagger}$ & $+8.4 \pm 5.3$ & $-0.5 \pm 0.5^{\dagger}$ & $-16.3 \pm 2.6$ & $+22.2 \pm 3.9$ \\
ETTh2   & $+50.1 \pm 11.6$ & $+54.8 \pm 40.4$ & $+4.8 \pm 3.2$ & $+1.5 \pm 5.6^{\dagger}$ & $-0.1 \pm 3.1^{\dagger}$ \\
ETTm1   & $-16.7 \pm 3.5$ & $-17.2 \pm 4.3$ & $-5.8 \pm 1.6$ & $-21.2 \pm 3.2$ & $+35.7 \pm 6.0$ \\
ETTm2   & $+17.9 \pm 6.2$ & $+54.6 \pm 22.8$ & $+3.0 \pm 1.7$ & $+0.4 \pm 1.7^{\dagger}$ & $+3.6 \pm 4.6^{\dagger}$ \\
ECL     & $-7.3 \pm 3.0$ & $+20.7 \pm 15.4$ & $-2.3 \pm 0.7$ & $-31.9 \pm 3.0$ & $+38.5 \pm 5.7$ \\
Traffic & $-11.2 \pm 3.0$ & $+19.5 \pm 16.2$ & $-1.0 \pm 0.2$ & $-39.9 \pm 2.4$ & $+49.8 \pm 6.1$ \\
Weather & $+1.2 \pm 4.0^{\dagger}$ & $+13.1 \pm 9.2$ & $-2.9 \pm 1.2$ & $+8.8 \pm 2.8$ & $+6.4 \pm 2.8$ \\
\bottomrule
\end{tabular}
\normalsize
\caption{95\% CIs for the mechanism columns of the main-text regime table, $t$-based over the same twenty-four backbone-horizon cell means as the means themselves ($\dagger$ marks intervals that include zero). The near-parity results are not separable from zero (RAFT on ETTh2 and ETTm2, SPM on ETTh1), while SPM's Weather win ($-2.9 \pm 1.2$) and its ECL and Traffic wins are separable from zero.}
\label{tab:paradigm_ci}
\end{table*}

\section{The Leakage Pitfall in Periodic-Copy Baselines}

A periodic copy implemented as $\hat{y}_{t+h} = y_{t+h-L}$ reads inside the forecast window for every $h > L$: at $H{=}336$ and $L{=}24$, 93\% of the outputs are copies of future ground truth, and the baseline's sign can flip on trend-bearing hourly data (on ETTh2 the leaky form scores $-15\%$ against the causal form's $+6\%$). We document this because the leaky form is a natural implementation accident: we caught it when a validation-selected copy lag collapsed to 2 on every smooth channel, which under the leaky indexing copies near-adjacent future values. All numbers in this paper use the causal tiling form.

\section{The Window Sweep with the Period Tile}

We score the period tile against the backbones retrained at each window, under the paired $\Delta$MSE\% convention at the two standard horizons ($H \in \{96, 336\}$), the setting used for every analysis beyond the headline tables. Table~\ref{tab:tilesweep} traces the result. As the window grows, the tile's benefit attenuates and crosses into a loss in the order of each dataset's period: ETTh1 between $S{=}12$ and $24$, ECL between $24$ and $48$, ETTm1 and Traffic between $48$ and $96$. ETTh2 and ETTm2 stay positive (the tile never helps) at every window. The crossings are consistent with each dataset's period structure, with ECL and Traffic confounded by the weekly copy lag their tile uses. This confirms on a second, mechanism-free arm that the boundary is a property of the operating point.

\begin{table}[htbp]
\centering
\tabstyle
\setlength{\tabcolsep}{5pt}
\begin{tabular}{llrrrr}
\toprule
\rowcolor{kilmist}  \hdr{Dataset} & \hdr{$L_p$} & \hdr{$S{=}12$} & \hdr{$S{=}24$} & \hdr{$S{=}48$} & \hdr{$S{=}96$} \\
\midrule
ETTh1   & 24  & $-9.1$  & $+18.1$ & $+29.2$ & $+32.5$ \\
ETTm1   & 96  & $-46.5$ & $-30.6$ & $-6.0$  & $+25.1$ \\
ECL     & 24  & $-33.6$ & $-10.0$ & $+4.5$  & $+24.8$ \\
Traffic & 24  & $-37.4$ & $-17.6$ & $-5.6$  & $+19.6$ \\
Weather & 144 & $+30.3$ & $+35.1$ & $+43.3$ & $+62.1$ \\
\bottomrule
\end{tabular}
\normalsize
\caption{Period tile paired $\Delta$MSE\% against the backbone retrained at each window (negative helps). The benefit attenuates and crosses zero in the order of $L_p$. ETTh2 and ETTm2 are positive at every window and omitted.}
\label{tab:tilesweep}
\end{table}

\section{The Diagnostic Across the Window Sweep}

Extending the per-cell table from the starved point to the full sweep gives 504 (dataset, backbone, horizon, window) cells: all four windows $S \in \{12, 24, 48, 96\}$ at $H \in \{96, 336\}$, plus the $S \in \{24, 48\}$ sweeps at their two additional horizons $H \in \{192, 720\}$, twelve (horizon, window) combinations over seven datasets and six backbones. Of these, 42\% benefit from the instrument. Sign prediction over ALL cells is weak for every feature set (best 0.62 against a 0.58 majority rule), which the regime account itself explains: cells near the boundary have effects near zero, where sign is noise. Restricting to actionable cells, leave-one-dataset-out accuracy of the concentration-plus-$S/L$ pair rises with the effect floor (0.69 at $|\Delta| \ge 1\%$, $n{=}267$; 0.73 at $|\Delta| \ge 2\%$, $n{=}172$, against majority rules of 0.58), with the margin spread across folds rather than concentrated in one. These accuracies carry a post-hoc selection caveat: the effect-size filter conditions on the measured outcome, and the cells it drops sit at the boundary (median $S/L$ of exactly 1.0 among dropped cells, against 0.25 among kept ones), so they are upper bounds a deployment cannot certify in advance, and the decision-cost summary below is the deployable quantity. The dataset-level ADF-plus-OOV pair does not extend beyond the starved point (0.51 on actionable cells): it answers whether a dataset's structure rewards memory at all, while the regime coordinates answer where along the window axis it stops doing so. In decision terms, deploying the instrument whenever the concentration-plus-$S/L$ rule predicts benefit gains $-0.21\%$ MSE per cell across the sweep, against $+0.76\%$ for always deploying and $0$ for never deploying. An oracle gains $-0.76\%$. Per-cell results ship in the code appendix.

\section{External Validation on Held-Out Datasets}

The leave-one-dataset-out estimate in the main text rotates the held-out fold through the seven development datasets. A stronger test freezes the diagnostic and applies it to data withheld from every stage of the study. We add two datasets never used to build the map, the instrument, or the diagnostic: Exchange (eight daily exchange-rate channels, a near-random-walk) and ILI (seven weekly influenza channels with a strong yearly period, $L{=}52$, so $S{=}12 \ll L$). Each is run under the identical protocol: base, the instrument (SPM), and RAFT across the six backbones, two horizons ($H \in \{96, 336\}$ for Exchange, $\{24, 36\}$ for ILI), and three seeds, 216 runs in all.

Their two diagnostic statistics, computed exactly as in distribution, place them at opposite ends of the regime axis: Exchange is non-stationary with a memory-hostile symbol distribution (ADF $p{=}0.36$, train-half OOV $0.80$), where the instrument is harmful ($+18.8\%$, helping in 4 of 12 cells) as is RAFT ($+3.4\%$); ILI is stationary with a well-covered vocabulary (ADF $p{=}0.04$, OOV $0.38$). The frozen logistic rule, the same ADF-plus-OOV model class that scores $0.76$ in distribution, trained here on all eighty-four development cells and applied to the twenty-four novel ones, predicts the instrument's per-cell sign at $0.62$ (15 of 24). The comparison that matters is against the in-distribution prior: ``retrieval usually helps'' holds for $0.69$ of development cells, but transferring that majority call to the novel data scores only $0.38$. The diagnostic corrects the direction the prior gets wrong.

Two qualifications keep the claim honest. First, the rule returns the same verdict, that the bare instrument does not net-help, for both datasets ($P_{\text{help}}{=}0.00$ for Exchange, $0.22$ for ILI). It is decisive and correct on Exchange and marginal on ILI, where the 12-step exact-lookup instrument is a coin-flip ($+1.7\%$, median $-0.7\%$, helping in 5 of 12 cells). Second, retrieval does pay on ILI, but through RAFT ($-6.8\%$), whose corpus windows reach the yearly period the instrument's short tuples cannot. On this genuinely new periodic dataset the retrieval machinery matters, a mild qualifier to the in-distribution finding that exact lookup matches graph diffusion. What transfers cleanly is the regime axis itself: the statistics correctly separate a dataset whose structure cannot reward memory from one whose structure can. Per-cell results and the analysis script ship in the code appendix.

\paragraph{Transfer to another plug-in family} A second transfer test holds the datasets fixed but changes the mechanism. The frozen ADF-plus-OOV rule, fit on the instrument's per-cell benefit signs, predicts RAFT's per-cell sign at only $0.68$ (leave-one-dataset-out $0.68$ as well), below RAFT's own $0.80$ majority: the per-cell rule does not transfer to corpus retrieval. Both errors are mechanistic. RAFT helps on trend-dominated ETTh2 (11 of 12 cells), which the ADF-led rule, tuned to the instrument's trend failure, rejects, because corpus retrieval extracts structure the short symbolic tuples cannot; and RAFT fails on Weather (helping in 1 of 12), which the rule accepts, because the instrument's per-channel independence handles Weather's heterogeneous phases while RAFT's whole-window retrieval cannot. The operating-point boundary is shared across the class (the mechanism sweep below), but the per-cell diagnostic is specific to the instrument it was fit on. The analysis script is in the code appendix.

\section{Generalization to a Retail Domain (M5)}

The seven benchmarks and the diagnostic share a domain: smooth, mostly hourly energy, traffic, and weather. To test whether the regime map is domain-general, we apply it unchanged to M5 (Walmart daily unit sales), a retail domain of very different character: intermittent count data (median 0.45 units per day, 73\% zero days) with a weekly rather than daily dominant period. We take all six aggregation levels (store$\times$department down to state, 123 series), the item totals, and a stratified sample of item$\times$store series spanning the concentration range, 9{,}109 series in all. The M5 data is public; the loader, feature script, and sweep ship in the code appendix.

\paragraph{The regime axis transfers} Concentration, the main-text spectral metric computed identically, predicts whether a cold-start copy helps at AUC 0.968 (Spearman $-0.57$ between concentration and the paired benefit of a weekly tile over a phase-free window mean at $S{=}5{<}L{=}7$, $H{=}7$, across the 9{,}109 series). The energy-benchmark threshold still recalls 0.93 of the helped retail series at balanced accuracy 0.90, though on count data it is over-inclusive. The crossing sits higher on count data, near 0.35, than on the energy benchmarks, near 0.13: intermittent noise demands more concentration before copying a specific day's count beats smoothing. Benefit by level tracks concentration exactly as the map predicts, from the concentrated store and department aggregates (concentration 0.37 to 0.44, copy helps) to the intermittent item$\times$store series (concentration 0.03, copy never helps).

\paragraph{The per-cell diagnostic does not} The fine-grained ADF-plus-OOV rule, by contrast, fails to transfer (AUC 0.31). Its features are calibrated to the symbolic predictability of smooth series; on intermittent counts they carry the opposite or no signal, with helped series showing higher rather than lower ADF $p$. The regime axis transfers across domains; the per-cell symbolic rule is domain-specific, the same instrument-specificity seen across plug-in families above, now across domains.

\paragraph{Trained replication} Running the SPM instrument against a DLinear backbone across $S \in \{3,5,7,14,28\}$ at $H{=}7$ on these series (124 aggregated levels and 60 stratified intermittent item series, $n{=}184$), the paired benefit tracks concentration at the starved window (Spearman $-0.45$ at $S{=}3$). The benefit is small, at most about $-1\%$ at $S{=}3$ and non-monotone across $S$ within three-seed noise, the modest margins expected for a short weekly period where starvation costs only a few phases, not a clean crossing. The decisive observation is on the intermittent item series: there the learned gate disengages, holding SPM within $0.5\%$ of the backbone where the learning-free tile loses 60\%. The instrument is robust exactly where the brittle copy is not.

\section{Exact Lookup versus Graph Diffusion}

Table~\ref{tab:exact} breaks the within-instrument comparison down by dataset. The per-cell verdict (every interval straddling zero) is summarized in the main text and tested in the significance section below.

\begin{table}[htbp]
\centering
\tabstyle
\begin{tabular}{@{}lrr lrr@{}}
\toprule
\rowcolor{kilmist} \multicolumn{3}{@{}c}{\hdr{\textit{ETT family}}} & \multicolumn{3}{c@{}}{\hdr{\textit{Other datasets}}} \\
\cmidrule(r{0.5em}){1-3} \cmidrule(l{0.5em}){4-6}
\rowcolor{kilmist} \hdr{Dataset} & \hdr{Exact $-$ diffusion} & \hdr{95\% CI} & \hdr{Dataset} & \hdr{Exact $-$ diffusion} & \hdr{95\% CI} \\
\midrule
ETTh1   & $+0.05$ & $\pm 0.30$ & Weather & $+0.24$ & $\pm 0.39$ \\
ETTh2   & $+0.28$ & $\pm 0.43$ & ECL     & $+0.02$ & $\pm 0.10$ \\
ETTm1   & $-0.12$ & $\pm 0.21$ & Traffic & $-0.05$ & $\pm 0.20$ \\
ETTm2   & $+0.07$ & $\pm 0.32$ & && \\
\bottomrule
\end{tabular}
\normalsize
\caption{Replacing the instrument's graph-diffusion retrieval with exact tuple lookup, as paired $\Delta$MSE\% within dataset, backbone, horizon, and seed (252 runs, 84 cells; positive means exact lookup is worse). Intervals are $t$-based 95\% CIs across cells. Every interval straddles zero.}
\label{tab:exact}
\end{table}

\section{Benefit Boundary Across Published Plug-ins}

The main-text boundary is swept on the symbolic instrument and the tile; the published plug-ins appear in the main-text regime table only at the single starved point $S{=}12$. To test that the operating-point boundary is a property of the mechanism class and not of our instrument, we sweep FAN, GTR, and RAFT across $S \in \{24, 48, 96\}$ on the five periodic datasets (six backbones, $H \in \{96, 336\}$, three seeds, 1{,}620 runs), paired against the backbones retrained at each window under the same convention as every other number in the paper. Table~\ref{tab:mechsweep} traces each plug-in's benefit from the starved point $S{=}12$ outward. Three TimeMixer cells diverged in training (paired MSE above five times the base; GTR-ECL at $S{=}24$ and $S{=}48$, RAFT-ETTm1 at $S{=}96$) and are dropped from the robust means. Where a plug-in helps at the starved point, its benefit attenuates as the window grows and crosses zero in the order of the period: FAN and GTR cross on ETTm1 between $S{=}24$ and $48$, and RAFT, the corpus method, crosses on ETTm1 between $S{=}48$ and $96$, at its period. RAFT's benefit on the short-period datasets (ECL, Traffic, ETTh1; $L_p{=}24$) attenuates monotonically but remains negative at $S{=}96$, consistent with corpus retrieval's native long-context regime. Of the eight (plug-in, dataset) pairs that help at $S{=}12$, six attenuate monotonically toward zero. The boundary is the operating point's, not the instrument's.

\begin{table}[htbp]
\centering
\tabstyle
\setlength{\tabcolsep}{4.5pt}
\begin{tabular}{llrrrrr}
\toprule
\rowcolor{kilmist}  \hdr{Plug-in} & \hdr{Dataset} & \hdr{$L_p$} & \hdr{$S{=}12$} & \hdr{$S{=}24$} & \hdr{$S{=}48$} & \hdr{$S{=}96$} \\
\midrule
FAN  & ETTm1   & 96 & $-16.7$ & $-7.4$  & $+3.3$  & $+4.8$ \\
FAN  & ECL     & 24 & $-7.3$  & $-2.5$  & $-0.6$  & $-1.3$ \\
FAN  & Traffic & 24 & $-11.2$ & $+3.4$  & $+5.6$  & $+1.8$ \\
GTR  & ETTm1   & 96 & $-17.2$ & $-2.9$  & $+0.3$  & $+13.1$ \\
RAFT & ETTh1   & 24 & $-16.3$ & $-7.5$  & $-4.5$  & $-3.1$ \\
RAFT & ETTm1   & 96 & $-21.2$ & $-18.9$ & $-11.1$ & $+11.5$ \\
RAFT & ECL     & 24 & $-31.9$ & $-19.0$ & $-12.6$ & $-9.4$ \\
RAFT & Traffic & 24 & $-39.9$ & $-30.1$ & $-25.9$ & $-15.9$ \\
\bottomrule
\end{tabular}
\normalsize
\caption{Paired $\Delta$MSE\% of each published plug-in as the window $S$ grows (negative helps), for the datasets where it helps at $S{=}12$. The sweep columns ($S \in \{24, 48, 96\}$) pair each plug-in against the backbones retrained at that window, six backbones at $H \in \{96, 336\}$, robust mean over backbone-horizon cells (the three diverged TimeMixer cells dropped, all at $S \ge 24$). The $S{=}12$ anchors are the four-horizon means of the main-text regime table; no $S{=}12$ cell diverges. The benefit attenuates as $S$ approaches $L_p$; FAN/GTR/RAFT cross zero on ETTm1, RAFT at its period. GTR, FAN on Weather/ETTh1, and RAFT on Weather do not help at $S{=}12$ and are omitted. Per-cell results are in the code appendix.}
\label{tab:mechsweep}
\end{table}

\section{Backbone Re-Tuning at the Starved Point}

The standard PatchTST configuration sets a patch length of 16, larger than the $S{=}12$ window, so at the starved point it cannot form a single full patch. To test that the tile's win is not an artifact of backbones carrying their standard-window configurations, we re-optimize the lookback-relevant hyperparameters of PatchTST (patch length in $\{2,3,4,6\}$, width, learning rate; 24 configurations) and iTransformer (width, depth, learning rate; 18 configurations) at $S{=}12$ on the three datasets where the tile wins most (ETTm1, Traffic, ECL), $H{=}96$, three seeds each. Table~\ref{tab:retune} reports, per pair, the test MSE of the single best re-tuned configuration, an oracle upper bound favorable to the backbone, against the default backbone and the period tile. Re-tuning improves the backbones by 3 to 32\% over their default configurations, yet the tile still beats the oracle-best re-tuned backbone on every pair, by 12 to 54\%. A window too short to contain one period is not repaired by tuning, consistent with the phase-free synthetic control.

\begin{table}[htbp]
\centering
\tabstyle
\setlength{\tabcolsep}{3.6pt}
\begin{tabular}{llrrrr}
\toprule
\rowcolor{kilmist}  \hdr{Backbone} & \hdr{Dataset} & \hdr{Default} & \hdr{Re-tuned} & \hdr{Tile} & \hdr{Tile adv.} \\
\midrule
PatchTST     & ETTm1   & 0.949 & 0.922 & 0.427 & $-53.7\%$ \\
PatchTST     & ECL     & 0.589 & 0.557 & 0.352 & $-36.8\%$ \\
PatchTST     & Traffic & 0.888 & 0.779 & 0.519 & $-33.3\%$ \\
iTransformer & ETTm1   & 0.846 & 0.755 & 0.427 & $-43.5\%$ \\
iTransformer & ECL     & 0.546 & 0.427 & 0.352 & $-17.4\%$ \\
iTransformer & Traffic & 0.867 & 0.592 & 0.519 & $-12.3\%$ \\
\bottomrule
\end{tabular}
\normalsize
\caption{Backbone re-tuning at $S{=}12$, $H{=}96$. Default is the published backbone configuration; Re-tuned is the best test MSE over the hyperparameter grid (oracle-best, favorable to the backbone); Tile is the period copy. Tile adv.\ is the tile's paired improvement over the re-tuned backbone (negative means the tile wins). Re-tuning helps the backbone by 3 to 32\% over its default, yet the learning-free tile still wins on every pair. Percentages are computed from unrounded MSEs.}
\label{tab:retune}
\end{table}

\section{Significance Tests}

Wilcoxon signed-rank tests over paired runs (one pair per backbone, seed, and horizon in $\{96, 336\}$; 36 per dataset) confirm the headline tile comparison. The four in-regime wins are all significant at $p < 10^{-8}$, with the tile better in 92 to 100 percent of runs (ETTm1, ECL, and Traffic at $p \approx 3 \times 10^{-11}$ and 100 percent; ETTh1 at $7 \times 10^{-9}$ and 92 percent). The three out-of-regime losses are significant in the opposite direction, with the tile better in only 0 to 17 percent of runs (Weather, ETTm2, ETTh2). Exact lookup versus graph diffusion over all 252 paired runs gives $p = 0.88$, consistent with no difference.

For the diagnostic, an exact McNemar test of the ADF-plus-OOV logistic rule against the majority rule over the 84 leave-one-dataset-out cells gives 9 versus 3 discordant cells, $p = 0.146$: the headline 0.76 versus 0.69 margin is suggestive, not significant, as the main text states. Table~\ref{tab:perfold} shows the two rules coincide on every fold except the trend-dominated ETTh2, where the diagnostic is right and the majority rule wrong. A fully prospective OOV variant, computed between training halves with the discretizer fit on the vocabulary half, reproduces the leave-one-dataset-out predictions exactly (0.76 accuracy, identical per-fold vector). Fitting the discretizer on the full training split instead drops accuracy to 0.62, so the prospective form requires that convention.

\begin{table}[htbp]
\centering
\tabstyle
\setlength{\tabcolsep}{6pt}
\begin{tabular}{@{}lrr lrr@{}}
\toprule
\rowcolor{kilmist} \multicolumn{3}{@{}c}{\hdr{\textit{ETT family}}} & \multicolumn{3}{c@{}}{\hdr{\textit{Other datasets}}} \\
\cmidrule(r{0.5em}){1-3} \cmidrule(l{0.5em}){4-6}
\rowcolor{kilmist} \hdr{Held-out fold} & \hdr{ADF $+$ OOV} & \hdr{Majority} & \hdr{Held-out fold} & \hdr{ADF $+$ OOV} & \hdr{Majority} \\
\midrule
ETTh1 & 0.58 & 0.58 & ECL & 0.92 & 0.92 \\
ETTh2 & \textbf{0.75} & 0.25 & Traffic & 1.00 & 1.00 \\
ETTm1 & 0.92 & 0.92 & Weather & 0.75 & 0.75 \\
ETTm2 & 0.42 & 0.42 &  &  &  \\
\bottomrule
\end{tabular}
\normalsize
\caption{Per-fold leave-one-dataset-out sign accuracy of the two-statistic rule against the majority rule, twelve cells per fold. The two coincide everywhere except the trend-dominated ETTh2 fold, which carries the entire headline margin.}
\label{tab:perfold}
\end{table}

\section{Protocol Details and Configurations}

The six backbones are the standard Time-Series-Library (TSLib) set and span the architecture families in current use: linear (DLinear), patch and inverted transformers (PatchTST, iTransformer), convolutional (TimesNet), MLP-mixer (TimeMixer), and exogenous-aware (TimeXer). The learned-cycle family (CycleNet, GTR) enters through the plug-in arm, where GTR, its most recent member, is run with both fixed and per-dataset tuned cycle lengths.

The protocol settings are as follows: backbone hyperparameters per (dataset, backbone) cell, SPM configuration ($V{=}5$ information-bottleneck bins fit on the training split, tuple order 3, mean-future exemplars, per-channel sigmoid gate initialized at 0.5, shared projection width 64), the diffusion variant (Personalized PageRank, teleport 0.05, top-3 successors), FAN frequency budget (top-4, selected by a DLinear validation probe on ECL and Traffic; top-8 infeasible at $S{=}12$), GTR cycle length 168 on all datasets (weekly on hourly data; 42 hours on 15-minute and 28 hours on 10-minute data), RAFT's published period set (three granularities) with our fixed top-20 retrieval budget, training budgets, and compute (CPU cluster nodes, 4 to 8 cores and 8 to 240 GB per run; PyTorch 2.4.1). Measured instrument overhead: across 35 paired (backbone, SPM) full-channel training runs, the median step-time ratio is 1.07 (90th percentile 1.39). Full command lines ship with the code appendix.

\section{Window-Length Theory: Detailed Comparison}
Concurrent with this work, \citet{butera2026context} prove that windows must exceed the generating process's memory length $P$ to attain minimum error. Their $P$ indexes identification ambiguity, not periodic structure. A deterministic period-$L$ cycle with distinct phase values has memory length far below $L$, so their necessity bound is loose exactly in the seasonal regimes the main text measures, where benefit persists until $S$ approaches $L$. Proposition~1 of the main text supplies the complementary bound for periodic structure: it lower-bounds window-only risk by the starved fraction $1 - S/L$ times a uniform lower bound on the waveform variance the covered phases leave undetermined. The two results then bracket the window question from the memory side and the phase side respectively.

\section{Non-Retrieval Related Work}

Deep forecasting backbones: DLinear \citep{zeng2023transformers}, PatchTST \citep{nie2023patchtst}, iTransformer \citep{liu2024itransformer}, TimesNet \citep{wu2023timesnet}, TimeMixer \citep{wang2024timemixer}, TimeXer \citep{wang2024timexer}. Scaling behavior of window length \citep{shi2024scaling}. Foundation models \citep{ansari2024chronos}.

\renewcommand{\thetheorem}{A.\arabic{theorem}}
\renewcommand{\theproposition}{A.\arabic{proposition}}
\theoremstyle{remark}
\renewcommand{\theremark}{A.\arabic{remark}}
\newtheorem*{mainpropA}{Proposition 1}
\newtheorem*{mainpropB}{Proposition 2}

\section{Theory: Statements and Proofs}
\label{app:theory}

This section proves the two results stated in the main text, the phase-starvation lower bound (Proposition~1 there) and the test-time risk bound (Proposition~2 there), and collects the supporting analysis of the instrument: a per-channel error decomposition and bounded-regression guarantee for the gate, consistency of the exact-lookup index under stationarity, and three structural properties of the discrete symbolic index. Restatements of the two main-text results are unnumbered and carry their main-text numbers; the appendix's own results carry A-prefixed numbers, and are stated for the instrument as presented in the paper, exact tuple lookup, with the graph-diffusion variant appearing only in Remark~\ref{rem:diffusion}. The supporting analysis motivates the design and the diagnostic. The operating-regime claim rests on the measurements, and Remark~\ref{rem:model-class} makes the model-class limitation precise.

\paragraph{Notation and standing assumptions}
Each channel $d \in \{1, \dots, D\}$ is discretized independently into $V{=}5$ information-bottleneck bins fit on the training split, mapping the value at time $t$ to a symbol $s^d_t \in \{1, \dots, V\}$. The symbolic state is the order-3 tuple $u^d_t = (s^d_{t-2}, s^d_{t-1}, s^d_t) \in \mathcal{U}_d := \{1, \dots, V\}^3$, so $|\mathcal{U}_d| \le V^3 = 125$. For every state $u$ observed in the training split, the index stores the payload $\bar{y}_u$, the empirical mean of the $H$-step future windows that followed occurrences of $u$. Given the discretized suffix $u^*$ of the current length-$S$ input window, the exact-lookup retrieval is $\hat{y}^{d}_{\mathrm{retr}} = \bar{y}_{u^*}$ when $u^*$ was observed in training, and otherwise the payload of the nearest observed state under a fixed symbol-space metric (the fallback). The gated fusion of the main text is
\begin{equation}
\hat{Y}^{d} = \big(1-\sigma(g_d)\big)\,\hat{Y}^{d}_{\mathrm{base}} + \sigma(g_d)\,\mathrm{MLP}\big(\hat{y}^{d}_{\mathrm{retr}}\big).
\end{equation}
For the analysis we absorb the projection into the retrieval branch and treat $\hat{y}^{d}_{\mathrm{retr}}$ as the fused auxiliary forecast. Targets are bounded, $|Y^d_{t+h}| \le B$ almost surely for all $h \le H$. Boundedness is an assumption, and after z-score normalization we read $B$ as an empirical bound. Every payload and every conditional mean below is then bounded by $B$. Statements are per channel and per forecast coordinate. The squared error averages over coordinates and sums over channels, so no generality is lost, and we drop the superscript $d$ inside proofs.

\subsection{The Gate: Decomposition and Bounded Regression}

\begin{proposition}[Per-channel error decomposition]\label{prop:decomp}
Hold the backbone, the index, and the projection fixed, and let $\mathcal{L}(\mathbf{g}) = \mathbb{E}\,\|\hat{Y} - Y\|^2$ as a function of the gate vector $\mathbf{g} = (g_1, \dots, g_D) \in \mathbb{R}^D$. Then
\begin{equation}
\mathcal{L}(\mathbf{g}) = \sum_{d=1}^{D} \mathcal{L}_d(g_d), \qquad \mathcal{L}_d(g_d) = \mathbb{E}\,\|\hat{Y}^d - Y^d\|^2,
\end{equation}
where each summand depends on $\mathbf{g}$ only through $g_d$. Consequently
\begin{equation}
\inf_{\mathbf{g} \in \mathbb{R}^D} \mathcal{L}(\mathbf{g}) = \sum_{d} \inf_{g_d \in \mathbb{R}} \mathcal{L}_d(g_d),
\end{equation}
and a single scalar gate $g$ shared across channels satisfies $\inf_{g} \sum_d \mathcal{L}_d(g) \ge \sum_d \inf_{g_d} \mathcal{L}_d(g_d)$, with equality when some common scalar attains the minimum of every $\mathcal{L}_d$.
\end{proposition}

\begin{proof}
The squared loss is a sum over channels of $\|\hat{Y}^d - Y^d\|^2$, and with the backbone, index, and projection fixed, $\hat{Y}^d$ depends on $\mathbf{g}$ only through $g_d$. Minimizing a separable sum decouples coordinatewise, which gives the first identity. The shared-gate program is the same minimization restricted to the diagonal $\{g_1 = \dots = g_D\}$, so its value is no smaller, and it equals the unconstrained value when the diagonal meets the product of the per-channel argmin sets.
\end{proof}

\begin{proposition}[Bounded regression of the gated predictor]\label{prop:gate-bound}
Assume finite second moments of $\hat{Y}^d_{\mathrm{base}}$, $\hat{y}^d_{\mathrm{retr}}$, and $Y^d$. Then for every channel $\inf_{g_d \in \mathbb{R}} \mathcal{L}_d(g_d) \le \mathcal{L}_d^{\mathrm{base}}$, where $\mathcal{L}_d^{\mathrm{base}}$ is the risk of the unmodified backbone on channel $d$, and hence $\inf_{\mathbf{g}} \mathcal{L}(\mathbf{g}) \le \mathcal{L}_{\mathrm{base}}$. On any channel where retrieval increases the risk for every positive gate value, the infimum is approached as $g_d \to -\infty$, where $\sigma(g_d) \to 0$ and the fusion recovers the no-plug-in backbone on that channel exactly.
\end{proposition}

\begin{proof}
Write $\sigma = \sigma(g_d) \in (0,1)$. Expanding the square, $\mathcal{L}_d$ is a quadratic polynomial $q_d(\sigma)$ with coefficients given by second moments of the three random vectors, and $q_d(0) = \mathcal{L}_d^{\mathrm{base}}$. Since $\sigma(\cdot)$ maps $\mathbb{R}$ onto $(0,1)$ and $q_d$ is continuous on $[0,1]$, $\inf_{g_d} \mathcal{L}_d(g_d) = \inf_{\sigma \in (0,1)} q_d(\sigma) = \min_{\sigma \in [0,1]} q_d(\sigma) \le q_d(0)$. Summing over channels and applying Proposition~\ref{prop:decomp} gives the system bound. If $q_d(\sigma) > q_d(0)$ for all $\sigma > 0$, the minimum over $[0,1]$ sits at $\sigma = 0$, approached along $g_d \to -\infty$.
\end{proof}

Here $\mathcal{L}_d^{\mathrm{base}}$ is the risk of the co-trained model's backbone branch with the gate closed, which need not equal the risk of the separately trained backbone used in the paired comparisons. In practice the gate is trained end to end with the rest of the model and the deployed parameters are fixed by validation-based early stopping, so the guarantee holds up to estimation error between the validation and test splits. The proposition licenses gate collapse as an escape that is available by design. The measured gate values (Measured Gate Values section of this appendix) show that trained gates in fact stay near their 0.5 initialization, and that the out-of-regime bound is instead enforced, loosely, by early stopping halting training before the retrieval branch grows large. The regime does not need to hold uniformly across channels for the bound to apply.

\subsection{Consistency of Exact Lookup under Stationarity}

\begin{theorem}[Consistency of the exact-lookup index]\label{thm:lookup-consistency}
Suppose the channel process $\{y_t\}$ is stationary and ergodic with targets bounded by $B$, and that its symbol sequence $\{s_t\}$ is a Markov chain of order 3 on $\{1, \dots, V\}$, equivalently the state sequence $\{u_t\}$ is a stationary finite-state Markov chain, irreducible on its reachable class $\mathcal{U}^* \subseteq \mathcal{U}$. Fix a query state $u^*$ with stationary probability $\pi(u^*) > 0$, and let $\bar{y}^{(n)}_{u^*}$ be the payload after a training split of length $n$. Then, almost surely as $n \to \infty$: (i) $u^*$ is eventually contained in the observed vocabulary, so the fallback is invoked only finitely often; and (ii)
\begin{equation}
\hat{y}_{\mathrm{retr}} = \bar{y}^{(n)}_{u^*} \;\longrightarrow\; f(u^*) := \mathbb{E}\big[Y_{t+1:t+H} \mid u_t = u^*\big].
\end{equation}
\end{theorem}

\begin{proof}
Stationarity and ergodicity of $\{y_t\}$ make the Birkhoff ergodic theorem applicable to bounded functionals of its trajectory, and both averages below are such functionals because the discretizer is a fixed measurable map of the window. Irreducibility on the finite reachable class gives the symbol chain positive stationary mass on every reachable state. Ergodicity alone drives the convergence; the Markov and irreducibility hypotheses serve the finite-sample and coverage results below. Applying it to $\mathbf{1}\{u_t = u^*\}$ gives
\begin{equation}
\frac{1}{n}\sum_{t \le n} \mathbf{1}\{u_t = u^*\} \;\longrightarrow\; \pi(u^*) > 0 \quad \text{a.s.},
\end{equation}
which forces infinitely many occurrences of $u^*$, proving (i). For (ii), apply the ergodic theorem to the bounded stationary sequence $\mathbf{1}\{u_t = u^*\}\, Y_{t+1:t+H}$ to get
\begin{equation}
\frac{1}{n}\sum_{t \le n-H} \mathbf{1}\{u_t = u^*\}\, Y_{t+1:t+H} \;\longrightarrow\; \pi(u^*)\, f(u^*) \quad \text{a.s.},
\end{equation}
the truncation at $n - H$ being a boundary effect that vanishes in the average. The payload $\bar{y}^{(n)}_{u^*}$ is the ratio of these two ergodic averages restricted to occurrences with full futures inside the split, so it converges almost surely to $f(u^*)$. On the event in (i) the lookup returns $\bar{y}^{(n)}_{u^*}$ for all large $n$, so $\hat{y}_{\mathrm{retr}}$ inherits the limit.
\end{proof}

\begin{remark}[Diffusion generalization]\label{rem:diffusion}
The graph-diffusion variant of the instrument replaces the single entry with a Personalized PageRank average
\begin{equation}
\hat{y}_{\mathrm{dbg}} = \sum_{v} \pi_{u^*}(v)\, \bar{y}_v, \qquad \boldsymbol{\pi}_{u^*} = \alpha \big(I - (1-\alpha)\widehat{P}_n^{\top}\big)^{-1} \mathbf{e}_{u^*},
\end{equation}
computed from the empirical transition matrix $\widehat{P}_n$ of the observed tuple graph (analyzed here without the top-$k$ truncation the implementation applies). For any row-stochastic $\widehat{P}_n$ the spectral radius of $(1-\alpha)\widehat{P}_n^{\top}$ is $1-\alpha < 1$, so the Neumann series gives
\begin{equation}
\boldsymbol{\pi}_{u^*} = \alpha \sum_{m \ge 0} (1-\alpha)^m \big(\widehat{P}_n^{\top}\big)^m \mathbf{e}_{u^*},
\end{equation}
a geometric mixture of multi-step transition distributions from $u^*$, with exact lookup as the $\alpha \to 1$ endpoint. The same ergodic argument as in Theorem~\ref{thm:lookup-consistency} gives $\widehat{P}_n \to P$ entrywise almost surely on states of positive stationary mass, the map $P \mapsto \boldsymbol{\pi}_{u^*}$ is continuous by the spectral-radius bound, and combining with payload convergence yields $\hat{y}_{\mathrm{dbg}} \to \sum_v \pi_{u^*}(v) f(v)$ almost surely. The diffusion variant is therefore consistent for a PPR-smoothed version of $f$ rather than for $f(u^*)$ itself, an additional smoothing bias that the ablation in the main text shows buys nothing on these benchmarks (largest per-dataset difference $0.28\%$, pooled Wilcoxon $p{=}0.88$), which is why the paper presents exact lookup as the instrument.
\end{remark}

\begin{remark}[Finite-sample rate]\label{rem:rate}
The almost-sure convergence admits a standard sharpening. Conditional on $N_u$ occurrences of a state $u$, each empirical transition frequency out of $u$ is an average of bounded increments, and Hoeffding's inequality gives
\begin{equation}
\Pr\big(|\widehat{P}_n(v \mid u) - P(v \mid u)| \ge \epsilon \;\big|\; N_u\big) \le 2\exp\big(-2 N_u \epsilon^2\big).
\end{equation} The payload coordinates satisfy the same exponential rate with $N_u$ deflated by an overlap factor, under one added assumption: the $H$-step future given the symbolic state screens off the earlier past, as when the process is Markov in its order-3 symbolic state. Futures following occurrences of $u$ that are fewer than $H$ steps apart share coordinates, and a blocking argument over occurrence groups separated by at least $H$ steps then yields conditionally independent increments of range $2B$. Without such screening, the blocking controls overlap but not serial dependence, and only the almost-sure limit of Theorem~\ref{thm:lookup-consistency} stands.
\end{remark}

\subsection{A Test-Time Risk Bound under Drift and Vocabulary Mismatch}

The consistency result concerns the training distribution. Deployment introduces two failure channels: the conditional law of the future given the symbolic state can drift between the train and test segments, and test-time states can fall outside the training vocabulary. Write $f_{\mathrm{tr}}(u)$ and $f_{\mathrm{te}}(u)$ for the conditional means of the $H$-step future given state $u$ under the train and test segments, $D_{\mathrm{TV}}(u)$ for the total-variation distance between the two conditional laws at $u$, and $\delta := \Pr_{\mathrm{te}}(u^* \notin \mathcal{V})$ for the out-of-vocabulary mass, where $\mathcal{V}$ is the set of states observed in training. To isolate the two deployment terms we evaluate the index at its training-population payloads, $\bar{y}_u = f_{\mathrm{tr}}(u)$. Theorem~\ref{thm:lookup-consistency} and Remark~\ref{rem:rate} control the finite-sample deviation from this idealization.

\begin{mainpropB}[Deployment risk of exact lookup; restated from the main text]
Under bounded targets, the exact-lookup retrieval satisfies
\begin{multline}\label{eq:drift-bound}
\mathbb{E}_{\mathrm{te}}\big[(\hat{y}_{\mathrm{retr}} - y)^2\big] \;\le\; \underbrace{\mathbb{E}_{\mathrm{te}}\big[(f_{\mathrm{te}}(u^*) - y)^2\big]}_{R_{\mathrm{sym}}\ \text{(Bayes-symbolic risk)}} \\
+\; 4B^2\, \mathbb{E}_{\mathrm{te}}\big[D_{\mathrm{TV}}(u^*)\big] \;+\; 4B^2\, \delta .
\end{multline}
\end{mainpropB}

\begin{proof}
The retrieval $\hat{y}_{\mathrm{retr}}$ is a deterministic function of the query state $u^*$ once the training index is fixed, and $f_{\mathrm{te}}(u^*) = \mathbb{E}_{\mathrm{te}}[y \mid u^*]$, so the cross term in the expansion of $(\hat{y}_{\mathrm{retr}} - f_{\mathrm{te}}(u^*) + f_{\mathrm{te}}(u^*) - y)^2$ vanishes by conditioning on $u^*$. This gives the exact decomposition
\begin{equation}
\begin{aligned}
\mathbb{E}_{\mathrm{te}}\big[(\hat{y}_{\mathrm{retr}} - y)^2\big] ={}& \mathbb{E}_{\mathrm{te}}\big[(\hat{y}_{\mathrm{retr}} - f_{\mathrm{te}}(u^*))^2\big] \\
&+ \mathbb{E}_{\mathrm{te}}\big[(f_{\mathrm{te}}(u^*) - y)^2\big].
\end{aligned}
\end{equation}
Split the first expectation on the out-of-vocabulary event $\{u^* \notin \mathcal{V}\}$. On that event the fallback payload and $f_{\mathrm{te}}(u^*)$ both lie in $[-B, B]$, so the integrand is at most $4B^2$ and the contribution is at most $4B^2 \delta$. On the complement, $\hat{y}_{\mathrm{retr}} = f_{\mathrm{tr}}(u^*)$, and for any conditional laws $P, Q$ on $[-B, B]$ the mean difference obeys
\begin{equation}
|\mathbb{E}_P y - \mathbb{E}_Q y| \le 2B\, D_{\mathrm{TV}}(P, Q),
\end{equation}
with $D_{\mathrm{TV}}$ the supremum over measurable events, so $(f_{\mathrm{tr}}(u^*) - f_{\mathrm{te}}(u^*))^2 \le 4B^2 D_{\mathrm{TV}}(u^*)^2 \le 4B^2 D_{\mathrm{TV}}(u^*)$ since $D_{\mathrm{TV}} \le 1$. Taking expectations and adding the three contributions proves the bound.
\end{proof}

The bound is stated at the level of the conditional future law because exact lookup predicts the future window directly. Drift in the one-step transition rows of the symbolic chain induces drift in the $H$-step future law, so transition-level and future-level drift statements are interchangeable up to constants that grow with $H$. The two non-Bayes terms of Equation~\eqref{eq:drift-bound} are the population quantities behind the pre-deployment diagnostic of the main text: the drift term $\mathbb{E}[D_{\mathrm{TV}}]$ grows when the process is nonstationary across the split, with trend domination as the dominant mode on these benchmarks, and its cheap training-split surrogate is the augmented Dickey-Fuller $p$-value (ADF); the mass $\delta$ is estimated directly by the symbolic out-of-vocabulary rate (OOV), the fraction of evaluation-time tuples absent from the index. When both statistics are small, the excess risk of the retrieval branch over the symbolic Bayes risk is small. The bound says nothing about whether the symbolic Bayes risk beats the backbone, which is precisely why it motivates the diagnostic and does not establish the operating-regime claim.

\subsection{Phase Starvation: Proof, Examples, and Scope}\label{app:phase}

\begin{mainpropA}[phase starvation; restated from the main text]
Let $y_t = \Phi(t \bmod L)$, the waveform $\Phi = (\Phi(0), \dots, \Phi(L{-}1))$ drawn from a prior with $\operatorname{Var}(\Phi(q) \mid \Phi(A)) \ge v > 0$ almost surely for every phase $q$ and every set $A$ of at most $S$ phases with $q \notin A$, where $\Phi(A)$ denotes the restriction $(\Phi(a))_{a \in A}$. Call a predictor window-only if it is a measurable function of $(y_{t-S+1}, \dots, y_t)$ alone. If $S < L$, every window-only predictor incurs risk at least $v$ on every horizon step whose target phase the window does not cover, hence at least $v\,(1 - S/L)$ per step averaged over any horizon $H = mL$ (integer $m$). A predictor with access to the realized waveform, which the training record supplies once it spans one period, incurs none of this term; at $S \ge L$ the term vanishes, and the period tile attains zero waveform risk.
\end{mainpropA}

\begin{proof}
Consider first the noiseless case and condition on the alignment $a = t \bmod L$, which is fixed or drawn independently of $\Phi$. Supplying the alignment only enlarges the predictor's information, so a bound proved with $a$ revealed lower-bounds every window-only predictor. Given $a$, the window is a deterministic function of $\Phi(P)$, where $P$ is the set of $\min(S, L)$ phases the window covers, and the horizon-$h$ target is $\Phi(q)$ with $q = (a + h) \bmod L$. For squared loss, any predictor measurable with respect to the window has risk at least the Bayes risk $\mathbb{E}[\operatorname{Var}(\Phi(q) \mid \Phi(P), a)]$. Coarsening from $\sigma(\Phi(P), a)$ to the window's $\sigma$-algebra only increases expected conditional variance, and the standing assumption bounds the Bayes risk below by $v$ whenever $q \notin P$ and $|P| \le S$. For the horizon average with $H = mL$, the target phases $(a{+}1) \bmod L, \dots, (a{+}H) \bmod L$ visit each residue exactly $m$ times and $L - S$ residues lie outside $P$, so exactly $m(L - S)$ of the $mL$ steps are starved and the per-step average is at least $v\,m(L-S)/(mL) = v(1 - S/L)$. For general $H$ the starved fraction deviates from $1 - S/L$ by at most $L/H$. Additive independent observation noise changes nothing: writing $\varepsilon_P$ for the noise on the covered coordinates, the noisy window is a function of $(\Phi(P), \varepsilon_P, a)$, and $\operatorname{Var}(\Phi(q) \mid \Phi(P), \varepsilon_P, a) = \operatorname{Var}(\Phi(q) \mid \Phi(P), a) \ge v$ since the noise is independent of $\Phi$, while target noise adds its variance to every predictor alike. If the record contains one full period of the realization, $\Phi$ is measurable with respect to it, so the predictor $\mathbb{E}[\Phi(q) \mid \text{record}]$ incurs none of the waveform term. If $S \ge L$, periodicity gives $y_{t+h} = y_{t + h - L\lceil h/L \rceil}$, a fixed coordinate of the window, so copying that coordinate, which is the period tile, attains zero waveform risk with no knowledge of $\Phi$.
\end{proof}

\paragraph{Example 1 (independent values)} If the phase values are independent with variance at least $v$, the assumption holds with that $v$: conditioning on other phases leaves each marginal untouched.

\paragraph{Example 2 (harmonic prior)} Let $\Phi(p) = \sum_{k=1}^{K} a_k \cos(2\pi k p / L) + b_k \sin(2\pi k p / L)$ with independent coefficient pairs drawn from a joint density (e.g., Gaussian), an idealization of the generator's random-phase harmonics. The generator matches the pairwise density requirement: each harmonic carries a random signed amplitude $\sim \mathcal{N}(0, 1/k^2)$ alongside its uniform phase, so each coefficient pair has a rotationally symmetric joint density on $\mathbb{R}^2$ rather than a fixed-amplitude law supported on a circle. The idealized part is independence across pairs, since the generator normalizes each sampled motif to unit empirical variance, rescaling all coefficients by one shared random factor. A global scaling reveals nothing about an unobserved phase value beyond what the $S$ linear constraints already carry, so the waveform stays underdetermined for $S < 2K$ and the conditional variance stays positive almost surely off the same degenerate configurations, though not uniformly bounded below. Example~1 is the prior that meets the assumption exactly. Observing $S$ phase values imposes $S$ linear constraints on the $2K$ coefficients, so for $S < 2K$ the waveform is underdetermined, and $\operatorname{Var}(\Phi(q) \mid \Phi(A)) > 0$ exactly when the evaluation functional of $q$ lies outside the span of the observed ones. Degenerate configurations exist, so this prior meets the proposition's assumption only off them: with $K{=}1$ and $L$ even, a single observation determines the antipodal phase, since $\Phi(p + L/2) = -\Phi(p)$. The per-step bound in the proof is local, so it applies verbatim at every starved step whose (window, target) configuration is nondegenerate, with $v$ the minimum conditional variance over those finitely many configurations. Example~1 supplies a prior meeting the assumption at every step. Once $S \ge 2K$ and the observed linear system is nonsingular, the coefficients, hence the waveform, are determined and the noiseless obstruction vanishes. The generator scales the harmonic count with $L$, which moves the determination threshold $2K$ up with the period, consistent with the measured boundary tracking $L$. The frozen-$K{=}2$ generator control removes the noiseless obstruction beyond $S \ge 4$, so its residual period-tracking is a noise-conditioning effect outside this proposition, reported in the generator-controls paragraph as sensitivity rather than claimed here.

\paragraph{Scope} Three limits are deliberate. The proposition bounds window-only prediction. A trained backbone is not window-only, since its weights have seen the training record, so its failure to close the gap at $S \ll L$ is measured, not derived. The empirical location of the crossing relative to $S = L$, and the role of noise where the noiseless obstruction is absent, remain empirical. And the bound concerns the deterministic-periodic idealization; real channels mix periodic and aperiodic structure, which is what the concentration statistic measures.

\begin{remark}[Model-class limit: no period-144 structure at order 3]\label{rem:model-class}
The order-3 symbolic state space shared by Theorem~\ref{thm:lookup-consistency} and the test-time risk bound contains at most $V^3 = 125$ states per channel. A deterministic symbolic orbit of period 144 occupies 144 phase positions per cycle over at most 125 available states, so by pigeonhole some state recurs at two different phases within a single period and the chain, conditioned on that state, cannot separate the two phases. Period-144 structure, the dominant period of Weather, is therefore strictly outside the model class, and the same holds for any period exceeding 125 symbol steps. The pigeonhole forces collisions on at least $144 - 125 = 19$ phases, not on all of them, so the index can still lower risk at the remaining phases against a starved backbone. This is consistent with SPM's small measured Weather gain, which rides per-channel independence rather than full period capture. The analysis above motivates the two diagnostic statistics. Beyond the phase-starvation lower bound restated above, it does not derive the precise regime boundary in $(S, L, H)$ nor certify benefit inside it. The operating-regime claim of the main text rests on the measurements, not on this appendix.
\end{remark}

\subsection{Structural Properties of the Symbolic Index}

Three properties distinguish the discrete exact-match index from continuous-similarity retrieval. They are preconditions for the instrument to be a clean probe inside the regime, not sufficient conditions for benefit: when the drift or OOV terms of Equation~\eqref{eq:drift-bound} dominate, coverage is intact but unhelpful, which is the failure mode the diagnostic detects.

\begin{proposition}[Noise absorption]\label{prop:noise}
Let $x, x' \in \mathbb{R}^S$ be two input windows on channel $d$ whose values fall in the same bin coordinatewise on the lookup suffix. Then the two windows produce the same query tuple $u^*$ and the retrieval is identical, $\hat{y}_{\mathrm{retr}}(x) = \hat{y}_{\mathrm{retr}}(x')$, for any within-bin perturbation. In contrast, any retrieval keyed by an injective continuous embedding of the window, for instance the identity embedding with Euclidean distance, assigns the two windows distinct keys whenever $x \ne x'$, so its retrieval is not invariant to within-bin noise in general.
\end{proposition}

\begin{proof}
The discretizer depends on the window only through the bin membership of each coordinate, so equal memberships on the suffix yield equal symbol triples and hence the same query tuple; the lookup, including the fallback branch, is a deterministic function of the query tuple, which gives the first claim. For the second, injectivity gives a nonzero key distance $\|x - x'\| > 0$, so the continuous query moves under perturbations that the discrete key absorbs, and retrieval invariance fails in general.
\end{proof}

\begin{proposition}[Exact-match coverage]\label{prop:match}
Under the assumptions of Theorem~\ref{thm:lookup-consistency}, let $\mathcal{U}^*$ be the reachable class and $\mathcal{V}_n$ the vocabulary observed in a training split of length $n$. There exist constants $C < \infty$ and $\rho \in (0,1)$, determined by the chain, such that a test query $u^*$ distributed according to the stationary law satisfies
\begin{equation}
\Pr\big(u^* \notin \mathcal{V}_n\big) \le C \rho^n.
\end{equation}
In the idealization of independent tuple draws the constants are explicit:
\begin{equation}
\Pr\big(u^* \notin \mathcal{V}_n\big) \le |\mathcal{U}^*|\, (1 - p_{\min})^n, \qquad p_{\min} = \min_{u \in \mathcal{U}^*} \pi(u),
\end{equation}
with $|\mathcal{U}^*| \le 125$. When $u^* \in \mathcal{V}_n$ the lookup key matches exactly and the retrieval error reduces to the payload estimation error of Remark~\ref{rem:rate}, whereas a continuous-key index over an atomless window distribution has exact-match probability zero, so every retrieval answers from a perturbed key rather than the queried one.
\end{proposition}

\begin{proof}
Fix a reachable state $s$. By irreducibility on the finite class $\mathcal{U}^*$ there exist $m \in \mathbb{N}$ and $\epsilon > 0$ such that from every state the chain visits $s$ within $m$ steps with probability at least $\epsilon$. Finiteness of $\mathcal{U}^*$ lets one take $m$ and $\epsilon$ uniform over starting states and over $s$. Applying the Markov property across $\lfloor n/m \rfloor$ consecutive blocks of length $m$ gives $\Pr(s \text{ unvisited in } n \text{ steps}) \le (1-\epsilon)^{\lfloor n/m \rfloor} \le C_0 \rho^n$ with $\rho = (1-\epsilon)^{1/m}$ and $C_0 = (1-\epsilon)^{-1}$. A union bound over the at most $|\mathcal{U}^*| \le 125$ reachable states bounds the probability that any reachable state is unvisited by $C \rho^n$ with $C = 125\, C_0$, and since the stationary law charges only $\mathcal{U}^*$, the event $\{u^* \notin \mathcal{V}_n\}$ is contained in that event, which requires no independence between the test query and the training trajectory. Under independent draws the per-state miss probability is exactly $(1 - \pi(s))^n \le (1 - p_{\min})^n$ and the union bound gives the explicit form. The exact-match claim is immediate: an observed key is its own nearest neighbor, so the lookup returns the entry stored at $u^*$ itself. For a continuous key with atomless query and key distributions, the event of an exact tie has probability zero.
\end{proof}

\begin{proposition}[Finite-space coverage]\label{prop:coverage}
Under the same assumptions,
\begin{equation}
\mathbb{E}\big[\,|\mathcal{U}^* \setminus \mathcal{V}_n|\,\big] \le |\mathcal{U}^*|\, C_0\, \rho^n \longrightarrow 0 \quad (n \to \infty),
\end{equation}
with $C_0$ and $\rho$ as in the proof of Proposition~\ref{prop:match}. The index therefore converges to complete coverage of the reachable symbolic structure of the process, and its asymptotic risk is the within-symbol Bayes risk, the first term of Equation~\eqref{eq:drift-bound}, rather than a quantity growing with the size of the index.
\end{proposition}

\begin{proof}
Write $|\mathcal{U}^* \setminus \mathcal{V}_n| = \sum_{s \in \mathcal{U}^*} \mathbf{1}\{s \text{ unvisited in } n \text{ steps}\}$ and take expectations; each term is bounded by $C_0 \rho^n$ by the block argument above, and linearity of expectation gives the claim. With $\delta \to 0$ and no drift, Equation~\eqref{eq:drift-bound} collapses to the Bayes-symbolic term.
\end{proof}

Taken together, the three propositions explain the design of the instrument rather than its benefit: discrete keys make the query invariant to within-bin amplitude noise, a 125-state space per channel is covered geometrically fast, and a covered query is answered by the empirical conditional mean at the exact state. Whether that answer helps depends on the operating point $(S, L, H)$ and on the drift and OOV terms of the test-time risk bound, which is the division of labor between this appendix and the measurements of the main text.

\end{document}